\documentclass[11pt]{article}

\usepackage[preprint]{acl}

\usepackage{times}
\usepackage{latexsym}
\usepackage[T1]{fontenc}
\usepackage[utf8]{inputenc}
\usepackage{microtype}
\usepackage{inconsolata}
\usepackage{graphicx}
\usepackage{booktabs}
\usepackage{multirow}
\usepackage{amsmath,amssymb}
\usepackage{amsthm}
\newtheorem{definition}{Definition}
\newtheorem{proposition}{Proposition}

\usepackage{algorithm}
\usepackage{algpseudocode}
\usepackage{tikz}
\usetikzlibrary{positioning,arrows.meta,fit,backgrounds,calc}

\title{Structural Baselines Matter for Agent Failure\\Detection and Attribution}

\author{
  Ramesh Arvind Naagarajan\textsuperscript{1}
  \quad
  Stefan Streif\textsuperscript{1,2}
  \\
  \textsuperscript{1}Professorship for Automatic Control and System Dynamics
  \\
  Chemnitz University of Technology, Chemnitz, Germany
  \\
  \textsuperscript{2}Department of Bioresources
  \\
  Fraunhofer Institute for Molecular Biology and Applied Ecology
  \\
  Giessen, Germany
  \\
  Corresponding author:
  \texttt{stefan.streif@etit.tu-chemnitz.de}
}

\begin{document}
\maketitle

\begin{abstract}
An LLM agent using tools can fail silently. It stops before the task is done, raises no error, and answers fluently. Two kinds of monitor address this, one deciding whether a run failed and the other naming the step where failure became unavoidable. Much of the reported performance on both tasks comes from regularities that need no model. Failing runs are longer, and the step labelled decisive is usually the last. A rule that always names the last step is correct $67\%$ to $96\%$ of the time here, without reading the run at all. That ceiling reflects how the labels were written, not how agents break. On failures whose decisive step replay can confirm, it falls to near chance. Reporting gain over that rule therefore changes what a score means. Under that correction, methods tested here gain nothing on give-up failures, while several published systems keep a clear advantage on other splits. For detection, the agent's internal states beat transcript length and text similarity on the agent benchmark, and hold up where the text baseline does not. Precision stays low when failures are rare, so such a monitor should rank runs for review, not replace it.
\end{abstract}

\section{Introduction}
\label{sec:intro}

Large language models now act as agents that call tools, read and write external state, and coordinate with other agents inside software systems. Many deployments run them without human review on every task. A central reliability concern in this setting is the \emph{silent failure}. The agent leaves the task incomplete but raises no exception, returns no error code, and shows no latency spike. Such failures are hard to catch because the output can look normal and no unsafe action has to occur. Studies of multi-agent systems report high task failure rates. Many of these failures are specification and coordination breakdowns that leave no trace for standard monitoring tools to record \citep{cemri2025mast,whoandwhen2025}. A monitor must decide whether a run that looks complete actually succeeded and, if not, where it went wrong.

Attribution methods developed so far do not match this setting well. Most of them assume labelled data, with ground-truth task outcomes and per-step annotations of which action was wrong. Many are tested on partial traces rather than on full production-shaped trajectories \citep{whoandwhen2025,elephant2026,rethinkattr2026}. The task remains hard even with those labels. Reported step accuracies on the Who\&When benchmark start in the single and low double digits. Later methods increase these scores, but there is disagreement on how much of the reported gain reflects genuine improvements \citep{whoandwhen2025,agentracer2025,raffles2026}.

Silent-failure detection faces a related set of constraints. Supervised detectors that target these failures also rely on labelled trajectories. They reach high area under the ROC curve (AUROC) only when many failed runs are available for training \citep{silentfail2025,trajguard2026}. A deployed monitor typically has none of these labels. It must judge a run from signals the system already records, such as per-step activations, tool returns, and text transcripts.

The setting considered here removes labels from the monitor. The monitor sees a full trajectory but no task label and no per-step annotation. From that trajectory it computes two quantities. One is a score for whether the run failed. The other is a guess at the step where failure became unavoidable. Section~\ref{sec:bench} defines this setting, the monitor, and the $\tau$-bench corpus used to study it.

Both outputs require a reference point before they are useful. Failing runs differ from successful ones in ways that do not need a model. They tend to be longer. Transcript length alone already separates many of them. The labelled step also tends to fall at the end of the run. Most failures in this corpus are omissions, where the agent stops and hands the task back rather than taking a harmful action. A rule that always names the last step is therefore right most of the time here, without reading any content. These two regularities play the role here that hypothesis-only artifacts play in natural language inference \citep{gururangan2018annotation}.

Both monitoring scores are therefore reported as gain over the best matched content-blind rule rather than as raw accuracy, which is the skill-score form long used in forecasting \citep{brier1950}, and a released script computes that gain from the gold step indices a benchmark already publishes. Applying the standard to $14$ splits across five benchmarks and to $30$ method--baseline pairs separates the reported skill that survives the correction from the skill that does not. The agent corpus behind those results, SILENT-Bench, is built on $\tau$-bench \citep{taubench2024} and fixes the decisive step by counterfactual replay rather than by human annotation, which also makes it possible to report how often a failure has a decisive step at all. The detector used throughout is label-free, scoring a run by its distance from a reference built only from trusted successful runs, and it carries the comparison of internal activations with text embeddings under a single protocol.
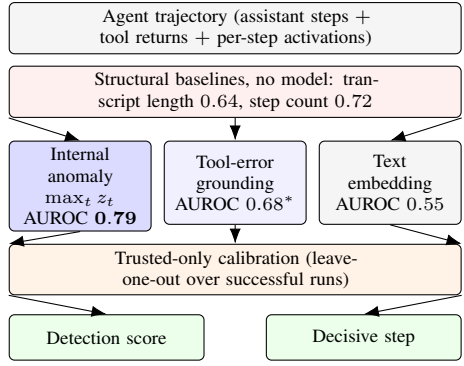
\begin{figure}[tb]
\centering
\begin{tikzpicture}[
  font=\scriptsize,
  box/.style={draw,rounded corners=2pt,align=center,inner sep=2.5pt,minimum height=6mm},
  sig/.style={box,fill=blue!6,text width=17mm,minimum height=11mm,align=center},
  outbox/.style={box,fill=green!8,text width=24mm,minimum height=6mm,align=center},
  >={Latex[length=2mm]}]
\node[box,fill=gray!8,text width=58mm] (traj) {Agent trajectory (assistant steps $+$ tool returns $+$ per-step activations)};
\node[box,fill=red!6,text width=58mm,below=1.5mm of traj] (base) {Structural baselines, no model: transcript length $0.64$, step count $0.72$};
\node[sig,below=3mm of base.south west,anchor=north west,fill=blue!14] (intern) {Internal\\ anomaly $\max_t z_t$\\ {\scriptsize AUROC $\mathbf{0.79}$}};
\node[sig,below=3mm of base.south,anchor=north] (tool) {Tool-error\\ grounding\\ {\scriptsize AUROC $0.68$\textsuperscript{$\ast$}}};
\node[sig,below=3mm of base.south east,anchor=north east,fill=gray!8] (text) {Text\\ embedding\\ {\scriptsize AUROC $0.55$}};
\node[box,fill=orange!10,below=2.5mm of tool,text width=58mm] (cal) {Trusted-only calibration (leave-one-out over successful runs)};
\node[outbox,below=2.5mm of cal.south west,anchor=north west] (det) {Detection score};
\node[outbox,below=2.5mm of cal.south east,anchor=north east] (attr) {Decisive step};
\draw[->] (base.south) -- (tool.north);
\draw[->] (base.south west) -- (intern.north);
\draw[->] (base.south east) -- (text.north);
\draw[->] (intern.south) -- (cal.north west);
\draw[->] (tool.south) -- (cal.north);
\draw[->] (text.south) -- (cal.north east);
\draw[->] (cal.south west) -- (det.north);
\draw[->] (cal.south east) -- (attr.north);
\end{tikzpicture}
\caption{The label-free pipeline, with detection AUROC on the regenerated corpus, rounded to two digits from Table~\ref{tab:detection}. Calibration uses trusted successful runs only, leave-one-out, so no failure label enters the monitor. \textsuperscript{$\ast$}tool grounding is measured on the greedy corpus.}
\label{fig:pipeline}
\end{figure}

\section{Related Work}
\label{sec:related}

\subsection{Step attribution in agent traces}

Who\&When introduces failed multi-agent traces with gold mistake steps and mistake agents, and reports step accuracies of $0.255$ and $0.088$ on its two splits \citep{whoandwhen2025}. AgenTracer trains a specialist tracer and finds that an $8$ billion parameter model reaches higher attribution accuracy than larger models on this benchmark \citep{agentracer2025}. Later benchmarks note that many attribution datasets assume one clear root cause and supply only partial traces. They also show that seeing the full trajectory can change attribution performance \citep{elephant2026,rethinkattr2026}.

Most of these methods rely on gold labels for both training and evaluation. More recent work follows the same pattern and adds components such as hypothesis verification, temporal modelling, or intervention-driven replay \citep{verifymas2026,stepfinder2026,dover2025}. Other studies argue that a single labelled step is often ill-posed \citep{whoandwhenpro2026,rethinkattr2026}. The monitor studied here operates without labels, and attribution scores are reported as gain over a positional baseline. That gain is a skill score, which expresses accuracy relative to a reference rather than on its own \citep{brier1950}. Earlier work on annotation artifacts found that the confound lay in the input text \citep{gururangan2018annotation}. For step attribution, the main confound instead lies in the position of the gold label.

\subsection{Silent failure detection and trajectory anomalies}

Some methods train supervised detectors on labelled trajectories and report high AUROC when enough failure labels are available \citep{silentfail2025,trajguard2026}. These systems need labels for both training and model selection. Other methods study unsupervised trajectory anomalies. They flag cycles or graph anomalies in multi-agent executions \citep{georgecycle2025,sentinelagent2025} or diagnose failures from execution traces \citep{trajad2026,agentrx2026}. Concurrent work characterising false success in agent runs reports that no configuration among five judges and five prompt strategies exceeds AUROC $0.65$ on $\tau^2$-bench \citep{falsesuccess2026}, which is a broader version of the judge result reported here from a single prompt on a different corpus. The goal in both cases is to catch failures that change the shape of a run. An omission barely changes that shape, since it looks like ordinary text and ordinary tool use, and Section~\ref{sec:exp-det} measures how far detectors built on repetition carry to it. The detector studied here instead compares each step with a reference built from successful runs alone.

\subsection{Activation geometry and out-of-distribution detection}

Directions in representation space can encode truthfulness and other behaviours \citep{zou2023repe,marks2024geometry}. Distance-based detectors such as Mahalanobis distance to class means or one-class support vector methods flag inputs outside a reference distribution \citep{lee2018mahalanobis,podolskiy2021revisiting,tax2004svdd,ruff2018deepsvdd}. Activations taken before generation can also predict task success when used with supervised probes \citep{agentprobe2026}. The detector here follows this distance-based line but adapts it to trajectories and to a setting with no labels. The reference distribution comes from successful runs only. A run is scored by the largest standardised distance among its steps.

\section{The label-free setting and SILENT-Bench}
\label{sec:bench}

A \emph{silent failure} is a run where the agent misses the task goal but produces no error or crash. A silent failure is an \emph{omission} when the agent simply stops, leaving the goal still reachable, and a \emph{harmful action} when some step makes the goal unreachable. The \emph{decisive step} is that step for a harmful action. For an omission there is no such step, and the label substitutes the give-up point by convention. For each run, the monitor returns a detection score and a guess at the step where the run went wrong. Trusted runs are successful runs that a deployment can recognise from downstream outcomes. The monitor is fitted, and its alert threshold is set, on those runs alone. Task labels enter only at evaluation time.

The corpus is built on top of $\tau$-bench, where a tool-using agent serves a simulated customer through a fixed airline or retail interface \citep{taubench2024}. Success is judged by the resulting database state, not by the conversation text. This makes the label deterministic. It also allows replay of a run to find the decisive step without ever showing that step to the monitor. The agent's activation is recorded at each turn and is taken from its own forward pass, so the monitor adds no extra inference calls. Appendix~\ref{app:bench} gives the splits, the model versions, and the release archive, and Appendix~\ref{app:vintage} lists which corpus supports which claim.

Replay computes the decisive step by applying the gold actions to a fresh database to get a target state, then finding the earliest agent write after which those actions can no longer reach it. Where no such write exists, the give-up point stands in. That is the hand-off to a human, or the last step where the agent used a tool. Replay discovers the first label, while the second is a convention that sits at the end of the run by construction. Figure~\ref{fig:failtypes} shows one of each from real runs, and Appendix~\ref{app:case} gives three worked examples. Both kinds end with no error and a fluent reply.

Most failures in this corpus are omissions. When every prefix of every failing run is checked, only $66$ of $596$ failures ($11.1\%$) have a decisive step that replay can verify (Appendix~\ref{app:gold}). The fraction varies by domain, from about $2\%$ on retail to nearly $40\%$ on airline. This mix makes distance from normal operation a usable detection signal and places the labelled step near the end of the run, which the positional baselines later exploit. The code and a tool that applies the positional correction to a new benchmark are released anonymously.\footnote{\url{https://anonymous.4open.science/r/acl_submission-5864}} The scoring code, the precomputed per-step and per-run scores, and the trajectory corpora are included with this submission as software and data, respectively; all materials will be made publicly available in the camera-ready version.

\section{Method}
\label{sec:method}

\subsection{Detection}
\label{sec:method-det}

The monitor first builds a picture of what a normal successful run looks like, then measures how far a new run sits from that picture. Only successful runs contribute to this reference, so the method uses no failure labels. Figure~\ref{fig:pipeline} shows the pipeline, and Appendix~\ref{app:algo} gives the full procedure.

Let $h_{t}$ be the activation of the agent at step $t$, read at the last prompt token before it acts. A single token is used rather than a pooled representation because that position is the one decision point every step shares, regardless of what the step then emits. Pooling would mix this decision state with the generated text that transcript baselines already cover. The reference is the mean $\mu$ and standard deviation $\sigma$ of these activations, taken over trusted runs only. The score of a step is its distance from that mean, measured in standard deviations and averaged over the $d$ dimensions, so scores stay comparable across models of different width (here $d$ is $3584$ for Qwen-7B, $4096$ for Llama-8B, $4608$ for Gemma-27B, $5120$ for Mistral-24B and Qwen-32B, and $8192$ for Llama-70B):
\[
z_{t} = \lVert (h_{t} - \mu) / \sigma \rVert_{2} / \sqrt{d}.
\]
This score uses all coordinates rather than a single direction, which is why it outperforms a single-direction variant (Section~\ref{sec:exp-det}).

A run is only as normal as its least normal step, so its score is $\max_{t} z_{t}$. An ablation shows that it is not the best pooling rule on most datasets, which makes the reported numbers conservative (Appendix~\ref{app:aggregator}). Two further signals stand in when activations are unavailable. One checks whether any tool call returned an error. The other measures how far the text of the run sits from trusted text, encoding the assistant turns with \texttt{all-MiniLM-L6-v2} and scoring cosine distance from the trusted centroid. Both are distances from trusted traffic rather than outputs of trained classifiers, since a classifier would need failure labels.

Each signal $s$ is standardised on trusted runs, and the results are added with equal weight:
\begin{equation}
\label{eq:detect}
D(\tau) = \sum_{s} \frac{s - \bar s_{\text{ben}}}{\hat\sigma_{\text{ben}}}.
\end{equation}
The sum runs over whichever signals a deployment has, and every signal is oriented so that larger means more anomalous. A larger internal distance, a tool call that returned an error, and a larger distance of the run's text from trusted text all raise $D(\tau)$. The weights are kept equal because tuning them would need labelled failures, and Section~\ref{sec:exp-det} shows that tuning adds little. A run is flagged when $D(\tau)$ exceeds the $(1-\alpha)$ quantile of the trusted scores, so $\alpha$ is the fraction of good runs that the monitor is allowed to flag. The main results use $\alpha=10\%$. Appendix~\ref{app:oppoint} gives the full curves. Every fitted quantity, including the threshold, depends on trusted runs alone. This is the sense in which the detector is label-free, stated and proved in Appendix~\ref{app:labelfree}.

\subsection{Attribution}
\label{sec:method-attr}

For attribution, the monitor must name a single step. The internal localizer chooses the step with the largest $z_t$. A tool localizer chooses the first step whose tool call returns an error. Both localizers are also compared against a GPT-4o judge, which sees the numbered run with tool names and returns a label \citep{openai2024gpt4o}.

If the labelled step is usually last, always guessing the last step scores well while reading no content, so three index-only rules serve as references, always-first, always-middle, and always-last, plus a uniform-random rule. A localizer counts as improved only when it beats the best of them.

\section{Experiments}
\label{sec:exp}

\subsection{Applying the correction to a new benchmark}
\label{sec:recipe}

The correction uses only the gold step indices that a benchmark already publishes, so it can apply to any step-attribution benchmark in four steps. First, for each run of length $n$, score three content-blind rules that predict steps $0$, $\text{round}((n-1)/2)$, and $n-1$, and also a uniform-random rule with accuracy $\mathbb{E}[1/n]$. Second, pick the best of the three positional rules on a random half of the runs and score it on the other half, averaged over many splits, removing the bias of choosing a rule on the data that judges it. 

Third, report the \emph{positional lift} $L = \text{Acc}(f) - A_{\text{pos}}$ for the method $f$ under test, where $A_{\text{pos}}$ is the accuracy of the best positional rule, together with a bootstrap interval over runs and a paired permutation test against the same positional rule. A content-blind rule scores $L=0$ by construction, so only a method that reads the run can stay positive. Fourth, when several methods or splits are reported together, apply the correction across them. A method earns credit only when $L$ stays positive after that correction. A script is released alongside the corpus that takes gold indices and predictions and returns $L$, its interval, and the test outcome.

\subsection{Detection}
\label{sec:exp-det}

Internal activations detect failures better than the structural and text baselines on this corpus. The test set is the regenerated corpus of Section~\ref{sec:bench}, with $660$ runs and $64$ trusted successes. Three tests were specified in advance, not revised afterwards, and Holm-corrected. All other analyses are exploratory (Appendix~\ref{app:stats}). A trusted run is never scored against a reference that contains it, to avoid inflating AUROC.

Transcript-length and step-count baselines are already strong and unstable. Pooled over domains they reach AUROC $0.640$ and $0.716$ (Table~\ref{tab:detection}), but on airline failing agents work longer before giving up ($0.740$) while on retail failing runs are slightly shorter ($0.456$). Each baseline is therefore fitted per domain, so that a detector does not look strong in one domain and weak in another for reasons unrelated to generalisation.

The internal signal remains strong under these checks, and the text embedding does not. The internal detector reaches AUROC $0.787$ with interval $[0.708, 0.857]$. This is higher than transcript length by $0.147$ and higher than the Sentence-BERT transcript embedding (SBERT) by $0.236$, with DeLong $p<0.001$ for both comparisons. A two-one-sided test rejects equivalence at both margins. Residualising on length leaves the internal detector at AUROC $0.679$ and SBERT at $0.543$.\footnote{Residualisation regresses each score on log transcript length by ordinary least squares over all runs and recomputes AUROC on the residuals.} The ordering is the same in both domains. On airline, the internal detector reaches AUROC $0.783$ and SBERT reaches $0.694$ ($p{=}0.044$). On retail, they reach $0.794$ and $0.612$ respectively ($p{=}0.038$).

The internal signal is read at the final residual layer, which is the strongest layer in a sweep across layers. The gap to the text baselines holds at that layer and across the sweep (Appendix~\ref{app:layer}).

Cycle and graph anomaly detectors designed for agent trajectories also score below the internal detector (Table~\ref{tab:detection}, DeLong $p<0.001$). An omission does not need to repeat anything, so these baselines have less to detect (Appendix~\ref{app:priorbaselines}). Deployments most often monitor an agent by asking a language model whether the run succeeded, and that does not help here either. A zero-shot Mistral-24B judge calls $99\%$ of runs failures, so its verdict is close to constant and says little on its own. A fairer test is the probability it assigns to that verdict, which separates the runs over $642$ distinct values. Even under that scoring the judge reaches AUROC $0.433$ with interval $[0.358, 0.501]$, so it does not rank runs above chance. Success here is a property of the final database state, and a successful run can end in the same polite hand-off as an omission (Appendix~\ref{app:judgedet}).

SBERT's pooled AUROC of $0.551$ is close to chance, yet the same signal reaches $0.805$ on the greedy corpus whose user side is written by GPT-4o. Under identical greedy decoding, swapping that simulator for a local open model moves the text embedding from $0.805$ to $0.547$ while the internal signal moves from $0.781$ to $0.810$ (Appendix~\ref{app:layer}). A text monitor's strength is therefore partly a property of which model writes the other half of the transcript, a dependence any evaluation on simulated users inherits. The internal signal shows no such dependence.

Many SILENT-Bench failures end with the agent explicitly handing the task to a human, so a fair objection is that the corpus mostly measures give-up detection. A keyword detector that fires on give-up phrasing reaches only $0.594$ (Table~\ref{tab:detection}), so the hand-off text alone does not explain detection performance. The truly silent subset goes further and drops every run that raises a tool error. Pooling the greedy and regenerated runs leaves $73$ such failures against $62$ successes. There the internal signal stays above chance at AUROC $0.624$ with interval $[0.523, 0.720]$, while SBERT sits at $0.582$ with an interval that includes chance (Appendix~\ref{app:layer}). For failures with no surface signal at all, that subset is the honest headline.

A label-tuned fusion of the same three signals reaches AUROC $0.862$ against $0.859$ for the equal-weight label-free variant, so failure labels add almost nothing here. A supervised probe on the activations also scores below the unsupervised distance (Appendix~\ref{app:skyline}). The leave-one-out fusion in Table~\ref{tab:detection} is $0.813$, the value a deployment should expect.

\begin{table}[tb]
\centering
\small
\resizebox{\columnwidth}{!}{%
\begin{tabular}{lccccc}
\toprule
Label-free signal & pooled & 95\% CI & resid. & airline & retail \\
\midrule
\multicolumn{6}{l}{\emph{structural baselines (no calibration)}} \\
Transcript length  & $0.640$ & $[.564,.723]$ & -- & $0.740$ & $0.456$ \\
Step count         & $0.716$ & $[.652,.777]$ & -- & $0.769$ & $0.585$ \\
Give-up keyword$^\ast$    & $0.594$ & $[.564,.627]$ & -- & -- & -- \\
\midrule
\multicolumn{6}{l}{\emph{prior trajectory anomaly detectors}} \\
Cycle detection, hybrid & $0.670$ & $[.592,.750]$ & -- & -- & -- \\
Graph rarity vs trusted & $0.619$ & $[.565,.670]$ & -- & -- & -- \\
\midrule
\multicolumn{6}{l}{\emph{zero-shot LLM judge (no calibration)}} \\
Judge, Mistral-24B$^\dagger$ & $0.433$ & $[.358,.501]$ & -- & $0.417$ & $0.521$ \\
\midrule
\multicolumn{6}{l}{\emph{calibrated signals (leave-one-out)}} \\
Tool-error grounding$^\ast$ & $0.684$ & $[.563,.805]$ & -- & -- & -- \\
Text, SBERT        & $0.551$ & $[.480,.624]$ & $0.543$ & $0.694$ & $0.612$ \\
Internal, RMS dist.& $\mathbf{0.787}$ & $[.708,.857]$ & $0.679$ & $0.783$ & $0.794$ \\
\bottomrule
\end{tabular}}
\caption{Detection AUROC, regenerated $\tau$-bench corpus ($660$ trajectories, $64$ trusted successes). \emph{resid.} residualises on log transcript length. $^\ast$scored on the greedy corpus ($165$ runs), the only vintage with a tool oracle. Every other row is the regenerated corpus. $^\dagger$scored by the probability the judge assigns to its own failure verdict, since its emitted verdict is near-constant (Appendix~\ref{app:judgedet}).}
\label{tab:detection}
\end{table}

\subsection{Detection on further tasks and families}
\label{sec:exp-gsm8k}

Only Llama-3.1 succeeds often enough on $\tau$-bench to calibrate a reference, and sampling more runs does not change this (Appendix~\ref{app:crossfamily}), so the family comparison moves to tasks other families can complete.

The same label-free protocol is run on GSM8K math problem solving \citep{gsm8k} with a calculator tool. Success is a correct answer, and a silent failure is a confident wrong answer. GSM8K is easier than $\tau$-bench, so every family produces enough successful runs to build a trusted reference (Llama-3.1-8B $29\%$, Qwen2.5-7B $63\%$, Mistral-Small-24B $78\%$, Qwen2.5-32B $85\%$).

Transcript length behaves similarly across tasks. It reaches AUROC between $0.46$ and $0.82$ over eleven test sets, for both give-ups and wrong answers, and under both the GPT-4o and the fully open user simulator.

The earlier tie between internal activations and text embeddings was due to limited data. On the original corpus, with $21$ trusted successes, the two scores differed by $0.024$, and their intervals were about $0.2$ wide. The largest arithmetic test set separates them in the same direction. On GSM8K with Qwen2.5-7B, which has $852$ trusted successes, the internal signal reaches AUROC $0.733$ and the text embedding reaches $0.579$. On $\tau$-bench both gaps are significant at $p<0.001$ and survive length residualisation.

Two further arms test whether this pattern holds at larger scale, where successful runs are more common. Qwen2.5-32B on GSM8K, with $256$ trusted successes, gives an internal AUROC of $0.809$ and a text-embedding AUROC of $0.632$ ($p{=}0.007$). Mistral-Small-24B on FinQA, with $97$ trusted successes, gives $0.760$ for the internal signal and $0.673$ for the text embedding ($p{=}0.048$). The internal signal therefore leads on six of the seven arithmetic arms in Table~\ref{tab:general}. Trusted-set size is not the only factor. GSM8K with Mistral-Small-24B also has a large trusted set ($233$ successes), yet the internal signal sits slightly below the text embedding, $0.63$ against $0.67$ ($p{=}0.314$). A small trusted set can hide a gap. A large one usually, but not always, reveals it.

Absolute AUROC is lower on the arithmetic tasks than on $\tau$-bench. A wrong answer moves the decision state less than a give-up. Success rates across these arithmetic arms run from $29\%$ to $85\%$, without a matching spread in AUROC.

A third task is FinQA, numerical reasoning over financial tables and text, where a calculator-tool agent answers questions under the same protocol \citep{chen2021finqa}. For Llama-3.1-8B and Qwen2.5-7B the internal signal again sits above lexical TF–IDF and slightly above the semantic embedding, at AUROC $0.73$ and $0.70$ under the leave-one-out calibration of Table~\ref{tab:general}, although without reaching significance. The Holm correction in Section~\ref{sec:exp-det} applies to three tests on $\tau$-bench that were specified in advance. The GSM8K and FinQA comparisons were not, and receive no multiplicity correction, so neither the significant Qwen GSM8K gap nor the null FinQA cells carry confirmatory weight (Appendix~\ref{app:stats}). An open-simulator variant of $\tau$-bench shows the same ordering (Appendix~\ref{app:opensim}).

\subsection{Attribution}
\label{sec:exp-attr}

The same omission structure that helps detection makes naming the failing step hard. The labelled step often lands at the end of the run, and that position alone can carry a high score.

Figure~\ref{fig:goldpos_main} (Appendix~\ref{app:goldpos}) shows where the gold step falls across benchmarks. It sits at the very end on the agent corpus, and mid-trajectory or early on the external ones.

A rule that reads no content can already do a lot on omission-heavy splits. On the retail split, always naming the last step is right $96\%$ of the time. No localizer tested here comes close to that value, yet this rule has learned nothing about the run. Raw step accuracy therefore mostly reflects where the gold label tends to sit. Table~\ref{tab:attribution} gives each localizer against the GPT-4o judge and the positional rules, with bootstrap intervals.

Localization is instead scored by the positional lift $L$ of Section~\ref{sec:recipe}, where for a run of length $n$ the positional family predicts $\text{round}(n p)$ for $p \in \{0, \tfrac{1}{2}, 1\}$ and $A_{\text{pos}}$ is its best accuracy. A normalised version dividing $L$ by the room left above the baseline is unstable as that room vanishes, so it is reported with a floored variant in Appendix~\ref{app:pss}.

\begin{table}[tb]
\centering
\small
\resizebox{\columnwidth}{!}{%
\begin{tabular}{lccccccc}
\toprule
Benchmark / split & $n$ & pos. & best pos. & learned & SOTA & lift & lift$^{\text{ho}}$ \\
\midrule
$\tau$ 8B airline    & $34$  & $.98$ & last $0.88$  & $0.53$ & --    & $-0.35$ & $-0.36$ \\
$\tau$ 8B retail     & $110$ & $1.00$ & last $0.96$  & $0.12$ & --    & $-0.85$ & $-0.85$ \\
$\tau$ 70B airline   & $39$  & $.84$ & last $0.67$  & $0.31$ & --    & $-0.36$ & $-0.36$ \\
$\tau$ Qwen airline  & $47$  & $.66$ & last $0.15$  & $0.32$ & --    & $+0.17$ & $+0.22$ \\
$\tau$ Gemma airline & $27$  & $.83$ & last $0.41$  & $0.44$ & --    & $+0.04^\ddagger$ & $+0.04^\ddagger$ \\
W\&W hand-crafted    & $58$  & $.46$ & last $0.12$  & $0.17$ & $0.28$ & $+0.05^\ddagger$ & $+0.07^\ddagger$ \\
W\&W algorithm-gen   & $126$ & $.37$ & first $0.16$ & $0.30$ & $0.52$ & $+0.14$ & $+0.16$ \\
SWE-bench Verified   & $120$ & $.59$ & mid $0.05$   & $0.03$ & --    & $-0.03$ & $-0.03$ \\
\bottomrule
\end{tabular}}
\caption{Positional-lift audit of step attribution. \emph{pos.} mean gold relative position, \emph{lift} learned minus best positional baseline, \emph{lift}$^{\text{ho}}$ with the rule chosen on a held-out half, \emph{SOTA} best published accuracy \citep{raffles2026}. $^\ddagger$the interval on skill touches zero (permutation $p\ge0.25$), so these cells are null results rather than small gains. The normalised skill score, with intervals, is in Appendix~\ref{app:pss}.}
\label{tab:attribution}
\end{table}

On omission-heavy splits, a simple positional rule already covers most of the apparent performance. The always-last baseline reaches $0.67$ on 70B airline, $0.88$ on 8B airline, and $0.96$ on 8B retail, above every learned localizer tested. The internal localizer reaches only $0.31$, $0.53$, and $0.12$ on these splits (Figure~\ref{fig:attr}, Table~\ref{tab:attribution}). A GPT-4o judge also falls below always-last on every split, since it tends to pick an early root cause rather than the late give-up point. GPT-4o also drives the user simulator behind these rows, but the positional baselines read no model output, and an open judge from another family behaves the same way. Mistral-Small-24B reaches $0.41$ on 8B airline and $0.05$ on 70B airline, far below always-last at $0.88$ and $0.67$.

Telling that judge that an omission's decisive step is the give-up point raises its accuracy to $0.85$ and $0.31$ and never hurts a trace, so most of its error is a mismatch in what counts as decisive rather than inability to read the run, though it still trails always-last on both splits (Appendix~\ref{app:judgepos}).

Five attribution settings come from SILENT-Bench, where omissions are common by design. Negative lift on these splits is therefore a property of that setting rather than of attribution in general, and within this setting the lift follows the gold labels rather than the method. The Qwen split fails more often by harmful writes (15 of 47), so always-last drops to $0.15$, and the judge shows a nominal lift of $0.17$. Two further cells are not distinguishable from their own baselines. On Gemma the skill interval is $[0.00, 0.20]$ with permutation $p{=}1.00$, and on Who\&When hand-crafted it is $[0.00, 0.14]$ with $p{=}0.25$. Their nominal lifts of $0.04$ and $0.05$ are therefore read as nulls. That leaves the Qwen split as the only SILENT-Bench cell with a positive lift, at $+0.17$, and it is not among the six pairs that survive the audit's correction (Appendix~\ref{app:stats}), so it stands as one suggestive cell rather than a demonstrated gain.

\subsection{How far does the confound reach}
\label{sec:exp-field}

The audit covers every offline benchmark that publishes a gold step index for each run, since a matched positional baseline needs these indices and cannot use benchmarks that report only aggregate accuracy. The final set has $14$ splits across five benchmarks, including $2221$ ProcessBench solutions \citep{processbench2024} whose gold step is the first wrong one. These labels fall early rather than late and would expose any standard that fires only on give-ups.

Section~\ref{sec:bench} reported that only $11.1\%$ of $\tau$-bench failures have a decisive step that replay can verify. For the rest, the gold label is a convention rather than a measurement, and that convention is what the positional baselines exploit. Scoring only the failures that do have a verified step moves the gold to the middle of the run, and the positional ceiling falls to near random (Appendix~\ref{app:wellposed}).

On TRAIL \citep{trail2025}, $81.1\%$ of $143$ traces carry more than one annotated error span (mean $4.01$)~(Appendix~\ref{app:trail}), so the labelled step is one choice among several.

The share of failures with a verified decisive step varies by agent, domain, and benchmark. Table~\ref{tab:recover} summarises this spread. Holding the environment fixed and changing only the agent takes the share from $35.3\%$ to $0\%$ (Fisher $p<0.001$). Holding the agent fixed and changing the domain takes it from $36.4\%$ to $2.3\%$. Across benchmarks it runs from $0\%$ to $87.4\%$. The positional ceiling tracks this share. Where few failures have a real decisive step, the labelled step is almost always the give-up point and an always-last rule is almost always right. The confound is therefore a property of the labelling convention rather than of how agents break, which makes this fraction something a benchmark should report. The newer $\tau^2$-bench \citep{tau2bench2025} enters only through recoverability. Too few and unrepresentative successful runs left no stable trusted reference for detection.

\begin{table}[t]
\centering
\small
\resizebox{\columnwidth}{!}{%
\begin{tabular}{llrrc}
\toprule
Benchmark / domain & Agent & $n$ & \% recov. & ceiling \\
\midrule
$\tau$-bench airline & Llama-3.1-8B      & $34$ & $35.3$ & $0.88$ \\
$\tau$-bench airline & Llama-3.1-70B     & $39$ & $25.6$ & $0.67$ \\
$\tau$-bench airline & Qwen2.5-7B        & $47$ & $25.5$ & $0.11$ \\
$\tau$-bench airline & Gemma-2-27B       & $27$ & $14.8$ & $0.40$ \\
$\tau$-bench airline & Mistral-24B       & $30$ & $0.0$  & -- \\
\midrule
$\tau$-bench airline & Llama-3.1-8B      & $154$ & $36.4$ & -- \\
$\tau$-bench retail  & Llama-3.1-8B      & $442$ & $2.3$  & $0.96$ \\
$\tau$-bench, both   & Llama-3.1-8B      & $596$ & $11.1$ & -- \\
\midrule
$\tau^2$-bench retail  & GPT-4.1         & $60$  & $63.3$ & -- \\
$\tau^2$-bench telecom & GPT-4.1         & $199$ & $87.4$ & -- \\
SWE-bench Verified     & CoderForge-32B  & $198$ & $7.6$  & $0.05$ \\
\bottomrule
\end{tabular}}
\caption{How often a failure has a decisive step that replay can verify. \emph{\% recov.} share of failures where some prefix makes the goal unreachable, \emph{ceiling} positional ceiling of that split. Methods and intervals in Appendices~\ref{app:wellposed} and~\ref{app:tau2}.}
\label{tab:recover}
\end{table}

Taking the ratio of the positional ceiling to random shows how strong this effect can be. Agent-run benchmarks average $1.9$ to $3.3\times$, while math step-localisation benchmarks sit near $1.1\times$ (Table~\ref{tab:fieldfull}). Who\&When shows the spread inside one benchmark (Appendix~\ref{app:whoandwhen}). Reported accuracy on its algorithm-generated split runs from $0.125$ to $0.516$, against an always-first baseline of $0.159$ \citep{whoandwhen2025,raffles2026}. The same benchmark thus supports both a below-baseline and a $3.2\times$ result, depending on the method.

The confound does not simply make every method look weak. Six of the $30$ method–baseline pairs clear their positional baseline at a $5\%$ false discovery rate, all of them math judges on ProcessBench, where a positional rule earns almost nothing (Appendix~\ref{app:fieldmethods}). On agent traces, no method run in this study is among those six. Two cells do show a positive nominal lift, the judge on the harmful-write Qwen split at $+0.17$ and the judge on the Who\&When algorithm-generated split at $+0.14$, but neither carries an interval that separates it from its baseline (Section~\ref{sec:exp-attr}). A published method run under its own official implementation also fails to clear, scoring $0.151$ against an always-first rule at $0.159$ on the protocol that reproduces its reported accuracy (Appendix~\ref{app:official}). An open reimplementation of a specialist architecture comes closest, clearing nominally but not after correction (Appendix~\ref{app:raffles}). Several published systems do clear, by up to $3.2\times$, while the original benchmark's own best result on one split falls below an always-last rule (Appendix~\ref{app:publishedlift}). The correction therefore rules out the appearance of skill in methods that never read content, without ruling out attribution skill itself. Growing the content-blind family to the best fixed index or a length-conditional rule raises the bar further on Who\&When and ProcessBench and trims the weakest nominal lifts, while the specialists that clear by a wide margin still clear (Appendix~\ref{app:pss}).

\subsection{Scale, deployment, and domain calibration}
\label{sec:exp-deploy}

Most of the detection evidence comes from one model family, and the 70B arm supports no scaling claim. With $11$ successful runs its AUROC interval $[0.710, 0.941]$ cannot resolve differences from the 8B arm below about $0.12$ (Appendix~\ref{app:bench}). Scoring takes $0.35$ to $1.14$ milliseconds per run on one CPU core with no extra inference pass, and calibration remains stable from $N{=}21$ down to $N{=}5$ trusted runs (Appendices~\ref{app:bensize} and~\ref{app:samplecomplexity}). Cross-domain AUROC stays close to within-domain values on the regenerated corpus, so the earlier collapse traces to the tiny retail trusted set rather than to the domain pair.

\section{Discussion}
\label{sec:discussion}

When activations are available, they are the signal to keep. They are the only signal that separates from transcript length on a corpus large enough to tell them apart. They remain useful when about $30\%$ of the calibration set is unlabelled failures, and keep their power under sampling, where the text signal drops to near chance (Appendices~\ref{app:layer} and~\ref{app:contamination}). The signal is spread out rather than stored in one place. Across eighteen sparse dictionaries almost no latent fires on failing steps while staying quiet on trusted ones, and the few that do separate no better than the best of an equal budget of random directions (Table~\ref{tab:sae}, Appendix~\ref{app:geometry}).

\begin{table}[h]
\centering
\small
\resizebox{\columnwidth}{!}{%
\begin{tabular}{lrrccc}
\toprule
Corpus & trusted & dict. & give-up & best latent & full \\
\midrule
GSM8K Qwen-7B     & $1866$ & $7168$  & $0$ & $0.61$ & $0.73$ \\
GSM8K Mistral-24B & $881$  & $10240$ & $0$ & $0.60$ & $0.64$ \\
$\tau$-bench 8B   & $261$  & $8192$  & $1$ & $0.63$ & $0.80$ \\
FinQA Llama-8B    & $165$  & $8192$  & $3$ & $0.65$ & $0.75$ \\
\bottomrule
\end{tabular}}
\caption{Sparse dictionaries fitted on trusted steps only. \emph{give-up} latents active on over $10\%$ of failing and under $1\%$ of trusted steps, \emph{best latent} best single-latent AUROC, \emph{full} full-space distance.}
\label{tab:sae}
\end{table}

This differs from work that finds single directions for truthfulness \citep{zou2023repe,marks2024geometry}, so a monitor should read the whole state rather than watch one unit.

Precision is the main limit, not ranking quality. At a $10\%$ false-positive rate the detector gives precision $0.25$ when $5\%$ of runs fail, and $0.06$ when $1\%$ do (Appendix~\ref{app:oppoint}), so most alerts are runs that succeeded. Because such a monitor could also quietly drop failing runs rather than escalate them, the alert stream should stay part of the audit trail.

\section{Conclusion}
\label{sec:conclusion}

On the omission-heavy splits a rule that names the last step and reads nothing else is correct 67\% to 96\% of the time. That is usually where the agent gave up. On failures whose decisive step replay can verify, the same rule falls to near chance. Both monitoring scores should therefore be reported as gain over the best matched structural baseline. For detection, that reference is transcript length, and for attribution it is the step position. Under that correction, no method tested here gains on those splits. A released script computes the attribution gain from the gold indices benchmarks already published. Internal activations detect failures better than transcript length and text embeddings once the test set is large enough to resolve the difference. The text baseline's strength depends partly on which model writes the user side. Precision stays between 0.06 and 0.25 at realistic failure rates, so such a monitor should rank runs for human review, not replace it.

\section*{Limitations and Future Work}

The study covers a small set of tasks and model families, namely airline and retail agents on $\tau$-bench, calculator agents on GSM8K and FinQA, and the Who\&When and TRAIL attribution benchmarks, across four model families \citep{llama3,qwen25,gemma2}. Whether the results carry to other agent stacks, tool interfaces, or long multi-agent workflows is untested. The cycle and graph baselines are reimplementations rather than the original systems, and they are scored without thresholds because the original thresholds are unpublished.

The method needs enough successful runs in each domain, and the reference does not transfer cleanly between domains or survive a change of agent. Each deployment therefore needs its own calibration and periodic rechecks. This requirement also prevented the use of a second detection environment. On $\tau^2$-bench, too few trusted successful runs, and unrepresentative ones, left the estimates unstable.

Detection results are strongest for runs that stop early. On GSM8K, where a silent failure is a confident wrong answer, AUROC is lower. Subtle harmful actions and coordination errors may require different signals.

The internal signal is shown to be distributed, but its structure is not explained. The dictionary and geometry results indicate that it is not a single feature or direction but do not specify what it represents.

\section*{Ethical Considerations}

This work uses only public benchmark data from $\tau$-bench, Who\&When, TRAIL, and the GSM8K calculator setting. No real user logs or personal data are involved, and no consent beyond the benchmarks' own licences is needed. Running a monitor of this kind on real logs would be different and would require standard data-protection and governance measures, including users' rights to access and deletion.

The monitor also carries a dual-use risk. It only observes and scores runs, and does not change what an agent can do. A detector that is cheap enough to run on every trajectory could be used to find and quietly drop failing runs that ought to be escalated, hiding failures instead of surfacing them. Deployments should therefore treat the alert stream as part of the audit trail, so that removing or filtering alerts is itself visible to oversight.

\bibliography{silent}

\appendix

\section{Data and labels}

\subsection{SILENT-Bench construction}
\label{app:bench}
The corpus instantiates the label-free setting of Section~\ref{sec:bench} on top of $\tau$-bench
\citep{taubench2024}. A GPT-4o simulated user drives the conversation. The agent under test is
Llama-3.1-8B-Instruct on airline and retail, and Llama-3.1-70B-Instruct on airline.

Each assistant turn is parsed into a numbered step holding either one tool call with its
deterministic return, or a plain text reply. The monitor therefore sees only what is available at
inference time. A balanced-brace JSON extractor handles both the native Llama tool-call format and
the \texttt{<function=>} format. An earlier version handled only the latter and mislabelled every tool call as a user response. After
the fix all runs were re-parsed and their $654$ executed tool returns checked against the environment
logs, so no misparsed run entered the corpus.

The decision state at each step is the residual stream at the final transformer layer, read at the
last prompt token immediately before the agent acts. This is layer $32$ at 8B and layer $80$ at 70B,
giving a $4096$-dimensional state at 8B. It comes from the agent's own greedy forward pass, with no
second inference call, and Appendix~\ref{app:layer} shows the final layer is the strongest. Tool
returns pass through a deterministic oracle that flags errors by string markers such as
\emph{error}, \emph{not found}, \emph{invalid} and \emph{denied}.

Table~\ref{tab:corpus} gives the composition. Generation ran on one NVIDIA H200 with the agent in
4-bit. A $500$-task pass at 8B takes about $35$ minutes and the full corpus a few GPU-hours, most of
it the 70B rollouts. Activations add no cost, and scoring at monitoring time is $0.35$ ms per run at
8B and $1.14$ ms at 70B.

\begin{table}[h]
\centering
\small
\begin{tabular}{llrrr}
\toprule
Agent & Domain & runs & succ. & fail \\
\midrule
Llama-3.1-8B  & airline & $50$  & $16$ & $34$ \\
Llama-3.1-8B  & retail  & $115$ & $5$  & $110$ \\
Llama-3.1-70B & airline & $50$  & $11$ & $39$ \\
\bottomrule
\end{tabular}
\caption{Greedy corpus composition. Tool-error rate is $0.44$ on airline and $0.72$ on retail, over $654$ executed tool returns.}
\label{tab:corpus}
\end{table}

\begin{figure}[t]
\centering
\begin{tikzpicture}[
  font=\scriptsize,
  >={Latex[length=1.4mm]},
  box/.style={draw,semithick,rounded corners=2pt,align=left,inner sep=3pt,
              text width=32mm,minimum height=16mm},
  outbox/.style={draw,semithick,rounded corners=2pt,align=center,inner sep=3pt,
              fill=gray!12,text width=68mm}]
\node[box,fill=blue!5,draw=blue!45,anchor=north west] (om) at (0,0) {%
  {\bfseries Silent omission} \emph{(airline)}\\[1pt]
  \ttfamily transfer\_to\_\allowbreak human\_agents\\
  \ttfamily $\to$ Transfer successful\\[1pt]
  \rmfamily\emph{no error, text distinctive}};
\node[box,fill=orange!6,draw=orange!55,anchor=north west] (ac) at (36mm,0) {%
  {\bfseries Acted failure} \emph{(retail)}\\[1pt]
  \ttfamily exchange\_...\allowbreak(\#W0000000)\\
  \ttfamily $\to$ Error: order not found\\[1pt]
  \rmfamily\emph{text normal, activation off}};
\coordinate (mid) at ($(om.south)!0.5!(ac.south)$);
\node[outbox,below=8mm of mid] (ban) {%
  \textbf{Silent failure}, no exception, fluent reply, task goal unmet};
\draw[->,blue!60,semithick] (om.south) -- (om.south|-ban.north)
  node[midway,left=0pt,font=\tiny]{text};
\draw[->,orange!70,semithick] (ac.south) -- (ac.south|-ban.north)
  node[midway,right=0pt,font=\tiny]{internal};
\end{tikzpicture}
\caption{Two silent-failure modes from the datasets, with real trajectory text.}
\label{fig:failtypes}
\end{figure}
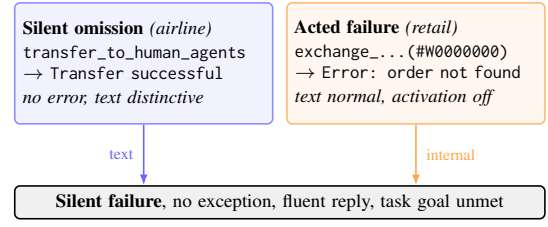

\subsection{Which corpus supports which result}
\label{app:vintage}

The \emph{greedy data} are $165$ runs with $21$ trusted successes, generated with a GPT-4o user simulator under greedy decoding. They include the tool oracle and the tool-grounding signal and are used for the layer sweep, contamination, drift, trusted-size, and supervised-skyline appendices, the 70B arm, and the attribution audit in Section~\ref{sec:exp-attr}. 

The \emph{regenerated data} are $660$ runs with $64$ trusted successes over airline and retail, generated with a fully open pipeline (a local Qwen2.5-32B user simulator) and sampling at temperature $0.7$ across four seeds. They are used for the detection results in Section~\ref{sec:exp-det}, the domain-transfer numbers, and the recoverability analysis in Appendix~\ref{app:wellposed}. Table~\ref{tab:vintage} shows which reported quantity comes from which dataset, so that comparisons across tables can be read with the source in view.

\begin{table}[h]
\centering
\small
\resizebox{\columnwidth}{!}{%
\begin{tabular}{lll}
\toprule
Quantity & Corpus & Where \\
\midrule
Detection AUROC, all signals & regenerated & Table~\ref{tab:detection} \\
Length and step-count baselines & regenerated & Table~\ref{tab:detection} \\
Tool-grounding signal & greedy & Table~\ref{tab:detection} \\
Give-up keyword row & greedy & Table~\ref{tab:detection} \\
Domain transfer & regenerated & \S\ref{sec:exp-deploy} \\
Recoverability fraction & regenerated & App.~\ref{app:wellposed} \\
Attribution audit, $\tau$ rows & greedy & Table~\ref{tab:attribution} \\
70B arm & greedy & \S\ref{sec:exp-deploy} \\
Layer sweep, drift, contamination & greedy & App.~\ref{app:layer} \\
Trusted-set size sweep & greedy & App.~\ref{app:bensize} \\
\bottomrule
\end{tabular}}
\caption{Which data vintage underlies which reported quantity.}
\label{tab:vintage}
\end{table}

\subsection{Decisive-step by counterfactual replay}
\label{app:gold}
Decisive-step labels come from deterministic counterfactual replay, which needs no GPU or API calls
because $\tau$-bench scores a trajectory by its final database state. The gold write actions are
applied to a fresh database to obtain a target hash, and the harmful decisive step is the earliest
agent write after which the gold actions can no longer reproduce it. Where no such write exists the
failure is an omission, and the label falls back to the give-up point.

Because the give-up point sits at or near the end of the trajectory by construction, an always-last
rule wins on omission-heavy splits partly through the labelling rule rather than through how methods
behave. The field-level claim in Section~\ref{sec:exp-field} therefore rests on external benchmarks,
whose gold labels are annotated independently of position. These splits show what a positionally
saturated benchmark looks like. They do not on their own show that attribution methods lack skill.

Under this rule most failures here have no decisive action at all. Evaluating recoverability at every
prefix, rather than only at write actions, gives $39$ of $257$ failures across the five splits, with
per-split rates in Table~\ref{tab:recover}. That omission structure is what makes distance from
normal operation a usable detection signal, and what places the labelled step at the end of the run.

\begin{table}[t]
\centering
\small
\resizebox{\columnwidth}{!}{%
\begin{tabular}{llccccc}
\toprule
Task (failure) & Model & $n$ (succ.) & length & SBERT & internal & $p$ \\
\midrule
$\tau$-bench (omission) & Llama-8B    & $660$ ($64$)  & 0.64 & 0.55 & $\mathbf{0.79}$ & $<0.001$ \\
GSM8K (wrong ans.)      & Llama-8B    & $100$ ($29$)  & 0.76 & 0.62 & 0.65 & $0.720$ \\
GSM8K (wrong ans.)      & Qwen-7B     & $1319$ ($852$) & 0.58 & 0.58 & $\mathbf{0.73}$ & $<.001$ \\
GSM8K (wrong ans.)      & Mistral-24B & $300$ ($233$) & 0.68 & 0.67 & 0.63 & $0.314$ \\
GSM8K (wrong ans.)      & Qwen-32B    & $300$ ($256$) & 0.67 & 0.63 & $\mathbf{0.81}$ & $0.007$ \\
FinQA (wrong ans.)      & Llama-8B    & $191$ ($58$)  & 0.73 & 0.70 & 0.73 & $0.508$ \\
FinQA (wrong ans.)      & Qwen-7B     & $191$ ($74$)  & 0.56 & 0.65 & 0.70 & $0.255$ \\
FinQA (wrong ans.)      & Mistral-24B & $191$ ($97$)  & 0.56 & 0.67 & $\mathbf{0.76}$ & $0.048$ \\
\bottomrule
\end{tabular}}
\caption{Detection across tasks and model families, leave-one-out calibration. \emph{length} is the transcript-length baseline. $p$ tests internal against SBERT.}
\label{tab:general}
\end{table}

\subsection{Cross-family airline splits}
\label{app:crossfamily}

Two further airline splits test whether the attribution audit depends on the Llama family. Qwen2.5-7B-Instruct is run through the same harness as the Llama agents, giving $50$ runs ($3$ successful, $47$ failing, $3584$-dimensional decision states). Gemma-2-27B-it needs a modified harness because its chat template rejects the system role and ignores the tools argument. The wiki and a text-rendered tool catalogue are merged into the first user turn. The conversation uses only user and model roles, and tool returns are presented as user turns. The model's JSON tool calls are parsed with the same balanced-brace extractor. The saved traces use the standard schema, so every analysis script runs unchanged. This split gives $30$ runs ($3$ successful, $27$ failing, $4608$-dimensional states).

Counterfactual replay marks $15$ of the $47$ Qwen failures and $4$ of the $27$ Gemma failures as harmful actions rather than omissions. On these splits the always-last baseline is therefore weak, and they are the two SILENT-Bench settings where a judge earns positive lift. Both splits enter the attribution audit only. With three successful runs each, they are too small to build a trusted reference for detection, and any detection number from such a small calibration set would be misleading. Sampling more runs does not change this. A further $55$ Qwen airline runs at temperature $0.7$ add only three more successes, a rate near $5\%$. Matching the $64$ trusted successes of the regenerated Llama test set would then need roughly a thousand Qwen rollouts on a domain with $50$ distinct tasks. The bottleneck is the agent’s success rate on this domain, not the number of runs sampled, so cross-family detection evidence comes from GSM8K and FinQA instead.

\subsection{Open-source pipeline and a new agent family}
\label{app:opensim}

The main dataset uses GPT-4o as the $\tau$-bench user simulator. To remove that external dependence and to add an agent family beyond Llama, the $\tau$-bench dataset is regenerated end to end with only open-weight models. A local Qwen2.5-32B-Instruct model drives the simulated customer through the same user-simulation loop, and a new agent family, Mistral-Small-24B-Instruct, is run on the airline domain. No API key is used at any point. Table~\ref{tab:opensim} reports label-free detection AUROC on these open-pipeline datasets under the same trusted-only protocol as in the main text.

The detection pattern holds on this fully open setup. On the new Mistral-Small-24B family, the two semantic signals both score above lexical TF–IDF, with SBERT at $0.71$ and the internal anomaly at $0.64$ against $0.40$ for TF–IDF. On a fully open Llama-3.1-8B dataset, the internal signal reaches AUROC $0.834$, close to the $0.842$ of the GPT-4o-simulated main dataset.\footnote{Two label-free calibrations appear for the pooled 8B greedy corpus. Values of $0.842$ (internal) and $0.857$ (SBERT) fit the trusted mean and deviation on all $21$ successes and score every run against that single reference. Table~\ref{tab:decoding} instead holds each successful run out of the reference before scoring it, which gives $0.781$ and $0.805$ on the same $165$ runs. The gap is the in-sample optimism of the first protocol, not a different signal.} The open Llama dataset has only $13$ successful runs across $150$ trajectories, which gives a smaller trusted set, wider confidence intervals, and a lower text-reference number. These results indicate that the internal detector remains stable under a fully open data-generation pipeline, but they are not combined with the main GPT-4o dataset.

\begin{table}[h]
\centering
\small
\resizebox{\columnwidth}{!}{%
\begin{tabular}{lcccccc}
\toprule
Open-pipeline corpus & $n$ & succ. & TF-IDF & SBERT & internal \\
\midrule
Mistral-24B airline        & $50$  & $20$ & 0.40 & 0.71 [.55,.85] & 0.64 [.48,.80] \\
Llama-8B airline$+$retail   & $150$ & $13$ & 0.52 & 0.68 [.51,.82] & 0.83 [.65,.98] \\
\bottomrule
\end{tabular}}
\caption{Detection AUROC on $\tau$-bench datasets from a fully open pipeline (Qwen2.5-32B user simulator), bootstrap $95\%$ CIs in brackets.}
\label{tab:opensim}
\end{table}

\section{Descriptive figures}

\subsection{Gold decisive-step positions}
\label{app:goldpos}

Gold decisive steps are not spread evenly along the trajectory. Relative position is defined as $g/(n-1)$ for a gold step $g$ in a trajectory of $n$ steps, indexed from zero. This quantity is undefined for single-step trajectories, where the only step is both the first and the last. Such trajectories are common in this dataset, with $14$ of $34$ failing airline runs and $13$ of $110$ failing retail runs consisting of one action. Assigning them position $0$, as a direct reading of the formula would do, makes them count as “earliest possible” while an always-last rule still predicts them correctly, which understates the mean gold position. Excluding them, the mean gold position is $0.98$ on 8B airline and $1.00$ on 8B retail, consistent with always-last accuracies of $0.88$ and $0.96$. The 70B airline split contains no single-step runs, and its mean position is $0.84$. The pooled agent-corpus mean behind Figure~\ref{fig:goldpos_main} is $0.98$, because the two 8B splits contribute $117$ of the $156$ pooled failures and their gold sits later. The judge analysis quotes $0.87$ on the same $39$ 70B failures because its step count comes from the gold annotation file rather than the trace parser, which shifts the denominator by one on some runs.

\begin{figure}[t]
\centering
\includegraphics[width=\columnwidth]{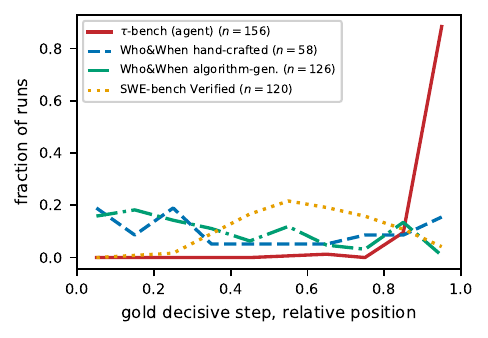}
\caption{Where the gold decisive step falls, as a fraction of the way through a run. On $\tau$-bench the mass sits at the end, which is what makes an always-last rule strong there. Single-step runs are excluded, since for them first and last coincide (Appendix~\ref{app:goldpos}).}
\label{fig:goldpos_main}
\end{figure}

Figure~\ref{fig:goldpos_main} shows this distribution across benchmarks, excluding single-step trajectories. On $\tau$-bench, omission gold sits at the very end of the run, so an always-last baseline is strong. On Who\&When, the hand-written gold sits near the middle of the trajectory (mean $0.46$), and the algorithm-generated gold sits earlier (mean $0.37$). On those splits, always-last and always-first are the key positional baselines. A positional baseline matched to each benchmark (always-last or always-first) is therefore strong on every benchmark. The positional-lift audit (Table~\ref{tab:attribution}) reports attribution performance relative to these matched baselines rather than in isolation.

\subsection{Decisive-step accuracy figure}
\label{app:attrfig}

Figure~\ref{fig:attr} shows decisive-step accuracy on the give-up gold across the SILENT-Bench splits. The always-last baseline (grey) scores higher than all learned methods and higher than the GPT-4o judge, even though the judge sees the full numbered run with tool names and tool outputs. This gap shows that raw step accuracy mostly reflects where the gold step tends to fall, not how well a method reads the steps. The main text therefore reports lift over this positional baseline rather than raw accuracy alone.

\begin{figure}[h]
\centering
\includegraphics[width=0.95\linewidth]{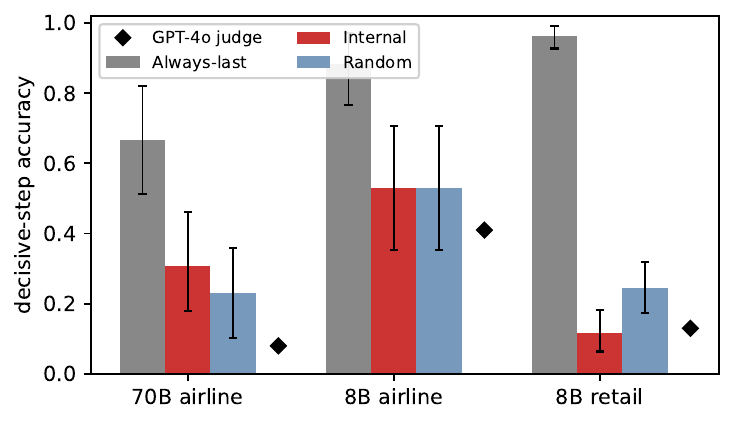}
\caption{Decisive-step accuracy on the give-up gold ($95\%$ bootstrap CIs). The always-last baseline (grey) dominates the learned localizers and the judge.}
\label{fig:attr}
\end{figure}

\subsection{Qualitative failure examples}
\label{app:case}

\textbf{A truly silent omission (70B airline):}
A representative omission runs four agent steps. The agent reasons, calls \texttt{get\_user\_details}, reasons again, and then calls \texttt{transfer\_to\_human\_agents}, handing the task back without performing the reservation change the task required. The trajectory carries no tool error and reads as a fluent, courteous hand-off. A lexical text monitor therefore cannot flag it. The counterfactual-replay gold marks the give-up step (the transfer, step~3) as decisive. The internal anomaly signal and the always-last rule both point near the end of the trajectory. This example shows why omission gold is mainly about position.

\textbf{An acted failure with hallucinated arguments (8B retail, task 1):}
The agent calls \texttt{exchange\_delivered\_order\_items} with identifiers that it never looked up, namely the placeholder order id \texttt{\#W0000000}, placeholder item ids, and the payment method \texttt{gift\_card\_0000000}. The environment answers \texttt{Error: order not found} and the agent gives up shortly after, in a five-step run. The surface text of the call reads like a normal exchange request. The trajectory-level embedding therefore stays near the trusted mean, while the per-step activation is strongly anomalous. This is the internal-only detection case of Appendix~\ref{app:disagree}.

\textbf{A single-step hand-off (8B airline, task 5):}
The agent's first and only action is \texttt{transfer\_to\_human\_agents}. The give-up text is lexically distinctive, so the embedding flags it easily. A one-step trajectory gives the internal detector little to measure, and its brief activation does not move far from trusted successful runs. This is the embedding-only detection case of Appendix~\ref{app:disagree}, and it shows why the fused monitor keeps both semantic signals where white-box access exists.

\section{Detection robustness}

\subsection{Layer sweep and multi-seed robustness}
\label{app:layer}

Detection AUROC is first checked across internal layers so that the layer choice in Section~\ref{sec:exp-det} is not arbitrary. Figure~\ref{fig:layer} reports detection AUROC on the pooled 8B airline and retail runs when the internal RMS distance is read at each of layers $4$ through $32$. The last layer (AUROC $0.842$) scores highest among the internal layers on the greedy data, where the SBERT baseline is $0.857$. On the regenerated data the ordering of the two signals reverses (Section~\ref{sec:exp-det}). The match between internal and text signals in Section~\ref{sec:exp-det} is therefore not just a side effect of which layer is read.

\begin{figure}[h]
\centering
\includegraphics[width=0.78\linewidth]{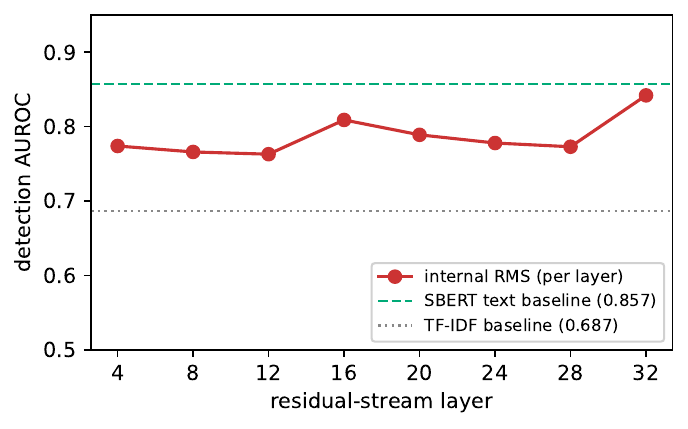}
\caption{Detection AUROC vs residual-stream layer, with text baselines marked.}
\label{fig:layer}
\end{figure}

Robustness to how trajectories are sampled is measured by regenerating the 8B airline split three times with the agent sampled at temperature $0.7$. The main runs use greedy decoding. Detection AUROC across the three seeds is $0.833 \pm 0.045$, in line with the greedy result. Pooling the three regenerations with the greedy runs also grows the truly silent subset (no tool error) to $73$ failures and $62$ successes. At that size the internal-anomaly AUROC is $0.624$ with interval $[0.523, 0.720]$, and the interval excludes chance. The SBERT text baseline is $0.582$ with interval $[0.481, 0.679]$, near chance on the noisier sampled text. The internal signal is therefore the more robust of the two on truly silent failures under stochastic decoding. Figure~\ref{fig:dist} shows how the internal anomaly score separates failing from successful trajectories on the 8B runs.

\begin{figure}[h]
\centering
\includegraphics[width=0.78\linewidth]{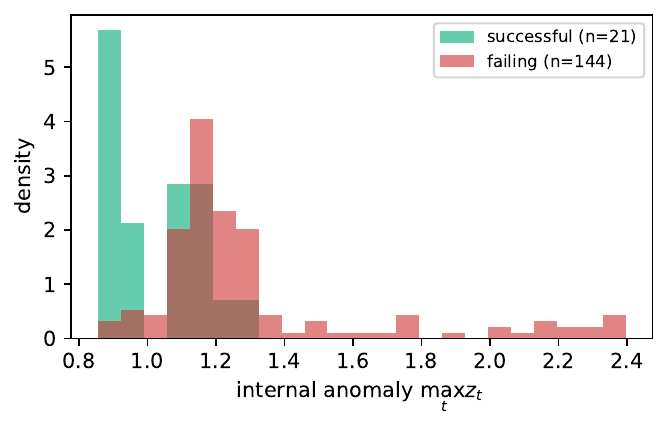}
\caption{Internal anomaly score, failing vs successful trajectories (8B).}
\label{fig:dist}
\end{figure}

\begin{table}[h]
\centering
\small
\resizebox{\columnwidth}{!}{%
\begin{tabular}{lccccc}
\toprule
Corpus & $n$ & succ. & internal & SBERT & $p$ \\
\midrule
\multicolumn{6}{l}{\emph{airline only, open simulator}} \\
greedy    & $50$  & $9$  & $0.799$ & $0.745$ & $0.455$ \\
temp $0.3$ & $200$ & $37$ & $0.824$ & $0.753$ & $0.073$ \\
temp $0.7$ & $200$ & $46$ & $0.783$ & $0.694$ & $0.044$ \\
\midrule
\multicolumn{6}{l}{\emph{airline$+$retail}} \\
greedy, GPT-4o sim & $165$ & $21$ & $0.781$ & $0.805$ & $0.589$ \\
greedy, open sim   & $150$ & $13$ & $0.810$ & $0.547$ & $0.005$ \\
temp $0.7$, open sim & $660$ & $64$ & $0.787$ & $0.551$ & $<0.001$ \\
\bottomrule
\end{tabular}}
\caption{Isolating decoding from the user simulator and from domain composition. The greedy GPT-4o row is scored leave-one-out, giving $0.781/0.805$ rather than the in-sample $0.842/0.857$.}

\label{tab:decoding}
\end{table}

Decoding is not the main factor, as shown in Table~\ref{tab:decoding}. On airline alone, where a three-point sweep is available, the internal advantage is $+0.054$ under greedy decoding, $+0.071$ at temperature $0.3$, and $+0.089$ at temperature $0.7$. The sign is the same in all three settings, and the magnitudes lie within noise of one another. Statistical significance follows the number of trusted runs rather than the temperature. Under identical greedy decoding, replacing the GPT-4o simulator with a local one moves the text embedding from $0.805$ to $0.547$, while the internal signal moves from $0.781$ to $0.810$. Who drives the user side therefore matters more than how the agent is decoded. Domain composition matters too, since pooled cells show larger gaps than airline alone and retail is where the text embedding is weakest.

\subsection{Prior trajectory anomaly detectors}
\label{app:priorbaselines}

Two families of unsupervised trajectory anomaly detector are cited in Section~\ref{sec:related} as targeting structural failures rather than omissions. That claim is tested here by reimplementing both families on the corpus and protocol of Table~\ref{tab:detection}. Cycle detection follows the two stages of \citet{georgecycle2025}. The structural stage scores a run by how often an identical tool call, matching on name and arguments, recurs, plus the number of repeated transitions between tool names. The semantic stage scores the largest similarity between two non-adjacent assistant steps under a TF–IDF representation. This is their redundant-content signal. The hybrid score standardises the two stages on trusted runs and sums them.

The graph detector follows \citet{sentinelagent2025} in scoring deviations from a normal profile. A transition graph over tool names is built from trusted successful runs alone. A run then scores by the fraction of its transitions that are absent from that graph plus its self-loop rate. Every trusted quantity is leave-one-out, as in the main protocol. Neither paper publishes an operating threshold, so all four detectors are scored threshold-free by AUROC, which is the most favourable reading available to them.

All four detectors score below the internal signal, and each gap is significant (Table~\ref{tab:priorbaselines}). The structural stage is the weakest at AUROC $0.570$. It produces a non-zero score on only $27.2\%$ of failing runs, against $14.1\%$ of trusted runs, so on nearly three quarters of the failures there is no repetition to detect. This matches the omission structure, since an agent that stops early and hands the task back does not loop. The semantic stage scores higher, at AUROC $0.660$, because give-up text is repetitive across steps and therefore closer to a text baseline than to a structural one.

\begin{table}[h]
\centering
\small
\resizebox{\columnwidth}{!}{%
\begin{tabular}{lccc}
\toprule
Detector & AUROC & $95\%$ CI & $p$ vs internal \\
\midrule
Cycle, structural & $0.570$ & $[.525,.613]$ & $<0.001$ \\
Cycle, semantic   & $0.660$ & $[.588,.733]$ & $<0.001$ \\
Cycle, hybrid     & $0.670$ & $[.592,.750]$ & $<0.001$ \\
Graph rarity      & $0.619$ & $[.565,.670]$ & $<0.001$ \\
\midrule
Internal, RMS dist. & $\mathbf{0.787}$ & $[.708,.857]$ & -- \\
\bottomrule
\end{tabular}}
\caption{Prior trajectory anomaly detectors on the regenerated corpus ($660$ trajectories, $64$ trusted successes), leave-one-out, threshold-free. DeLong tests are paired against the internal signal, whose row repeats the Table~\ref{tab:detection} interval.}
\label{tab:priorbaselines}
\end{table}

\subsection{A zero-shot LLM judge for detection}
\label{app:judgedet}

Deployments most often monitor an agent by asking a language model whether the run succeeded. To score that reference, Mistral-Small-24B-Instruct (4-bit, a different family from the Llama agent) is shown each full transcript of the regenerated corpus with roles marked, and returns a single 0--100 score that the run failed, under one fixed zero-shot prompt and greedy decoding. No calibration set is used, so the judge sees strictly more than the label-free signals, including the user's requests.

Under the $0$--$100$ prompt the emitted verdict saturates, returning $100$ on $640$ of the $660$ runs and $0$ on the remaining $20$. That flags $97\%$ of runs as failed, including all $64$ successes. Its AUROC of $0.483$ is therefore close to what any near-constant output gives by construction, and on its own it cannot support a claim about judge-based detection.

The judge is therefore re-scored without relying on the emitted text. The same model is asked the same question as a binary verdict, and the score is the probability it assigns to the failure answer, read from the first-position logits as a softmax over the yes and no tokens. This is continuous even when the argmax never changes. It separates the runs over $642$ distinct values, so the degeneracy of the text output is removed rather than measured. Under that scoring the judge reaches AUROC $0.433$ with interval $[0.358, 0.501]$, an interval whose upper end only just touches chance, and it assigns a slightly higher mean failure probability to successful runs ($0.955$) than to failing ones ($0.937$). The argmax verdict still calls $99.1\%$ of runs failures. A degenerate output format is therefore not what drives the result.

The cause is structural rather than a prompt accident. Success on $\tau$-bench is a property of the final database state, some tasks require no state change at all, and a successful run can therefore end in the same polite hand-off as an omission, so the transcript does not carry the verdict the judge is asked for. The result reflects this judge under one prompt rather than a ceiling for judge-based detection, and a judge given the task specification and tools to inspect the database would be solving a different and more expensive problem than the label-free monitor studied here. Concurrent work that varies both axes reports the same ceiling, with no configuration across five judges and five prompt strategies exceeding AUROC $0.65$ on $\tau^2$-bench \citep{falsesuccess2026}.

\subsection{Combining per-step scores into a run score}
\label{app:aggregator}

A run is scored by the largest standardised distance among its steps. This rule was specified in advance, and not revised afterwards, on the reasoning that a run is only as normal as its least normal step. Table~\ref{tab:aggregator} tests this choice against three alternatives on every corpus with recorded activations. These are the mean over steps, the mean of the three largest steps, and the last step alone.

The chosen rule is not the best on every corpus. The maximum is best or tied on only $3$ of $13$ datasets. The mean of the three largest steps beats it significantly on three datasets (with $p{=}0.000$, $0.004$, and $0.004$). The direction is consistent. Pooling a few high steps helps on the arithmetic datasets, where a wrong answer perturbs several steps mildly, while the maximum holds up on the agent datasets, where a single give-up step dominates.

The detection numbers in the main text are conservative, since a better aggregator would raise them on most datasets, by up to $0.04$ on the largest arithmetic set. The consequence is not only a smaller decimal. On the arithmetic datasets the aggregator can change which signal wins. On GSM8K with Mistral-Small-24B, the one arm where the internal signal trails the text embedding, the maximum gives AUROC $0.631$ against the embedding's $0.669$, while the last-step score gives $0.669$ and the mean gives $0.657$, so under a last-step aggregator that arm would be a tie rather than a loss. The count of six out of seven arithmetic arms in Section~\ref{sec:exp-gsm8k} is therefore conditional on the pre-specified maximum. The maximum is kept everywhere because it was fixed in advance, and selecting an aggregator after seeing which one repairs an unfavourable cell would not be a clean measurement.

\begin{table}[h]
\centering
\small
\resizebox{\columnwidth}{!}{%
\begin{tabular}{lrcccc}
\toprule
Corpus & $n$ & max & mean & top-3 & last \\
\midrule
$\tau$ 8B airline$+$retail & $165$ & $\mathbf{0.781}$ & $0.762$ & $0.765$ & $0.710$ \\
$\tau$ 70B airline        & $50$  & $0.834$ & $0.802$ & $\mathbf{0.851}$ & $0.816$ \\
$\tau$ 8B retail          & $115$ & $0.473$ & $0.440$ & $0.413$ & $\mathbf{0.482}$ \\
GSM8K Llama-8B            & $100$ & $0.645$ & $0.659$ & $\mathbf{0.681}$ & $0.646$ \\
GSM8K Qwen-7B             & $1319$& $0.733$ & $\mathbf{0.754}$ & $0.752$ & $0.718$ \\
GSM8K Mistral-24B         & $300$ & $0.631$ & $0.657$ & $0.655$ & $\mathbf{0.669}$ \\
GSM8K Qwen-32B            & $300$ & $0.809$ & $0.840$ & $\mathbf{0.853}$ & $0.656$ \\
FinQA Llama-8B            & $191$ & $0.726$ & $\mathbf{0.766}$ & $0.765$ & $0.758$ \\
FinQA Mistral-24B         & $191$ & $\mathbf{0.760}$ & $0.743$ & $0.749$ & $0.690$ \\
Open-sim $\tau$ 8B        & $150$ & $\mathbf{0.810}$ & $0.651$ & $0.725$ & $0.596$ \\
\bottomrule
\end{tabular}}
\caption{Detection AUROC by run-level aggregator, leave-one-out throughout. The maximum, which the main results use, is best or tied on 3 of the 13 datasets scored. Ten are shown, and the remaining three follow the same pattern. The 70B row rests on $11$ trusted successes and is underpowered (Section~\ref{sec:exp-deploy}).}
\label{tab:aggregator}
\end{table}

\subsection{Statistical protocol}
\label{app:stats}

Three tests carry the main claims, two comparing the internal signal against transcript length and against the text embedding, and a third checking by a two-one-sided test whether the two signals behave the same. Their $p$ values are adjusted together with a Holm correction. After this adjustment, both DeLong comparisons remain significant, with $p < 0.001$ against thresholds of $0.017$ and $0.025$.

All other tests are treated as exploratory and reported without adjustment. These include the per-domain detection splits, the per-dataset comparisons in Table~\ref{tab:general}, the attribution paired tests, the agent-gradient contrasts, and the $\tau^2$ comparisons. Several have $p$ values near $0.04$, including the airline and retail internal-versus-text splits and the Who\&When paired bootstrap, and none would survive a correction across all tests.

The attribution audit is the one place where many tests are read together, so it is adjusted. Each of the $30$ (method, benchmark) pairs is tested against its matched positional baseline using a paired permutation test over trajectories. The $p$ values from these tests are then adjusted across the whole audit with a Benjamini–Hochberg procedure. Eight pairs have a positive skill score whose bootstrap interval excludes zero. Six remain significant after adjustment at $q < 0.05$, and these six support the claim in Section~\ref{sec:exp-field}. The two that do not pass the $q < 0.05$ threshold, Mistral-24B on omnimath and on olympiadbench, already have $p > 0.05$ before adjustment.

\subsection{Trusted-set size sweep}
\label{app:bensize}

The internal reference set is down-sampled to $N$ successful trajectories drawn from the pooled 8B runs (21 successes in total). The per-dimension mean and standard deviation are re-estimated on that sample, and detection AUROC is recomputed, averaged over $50$ resamples per $N$. Figure~\ref{fig:bensize} shows that both signals fade in a smooth way as $N$ shrinks. Internal AUROC holds at $0.842$ with all $21$ successes, stays between $0.81$ and $0.84$ down to $N{=}5$, and falls only modestly to $0.803$ at $N{=}3$. The SBERT reference tracks it from $0.857$ down to $0.776$. A few tens of successful trajectories are therefore enough to set up the reference statistics, while very small trusted sets lose power and should be checked against a larger set of runs.

\begin{figure}[h]
\centering
\includegraphics[width=0.92\linewidth]{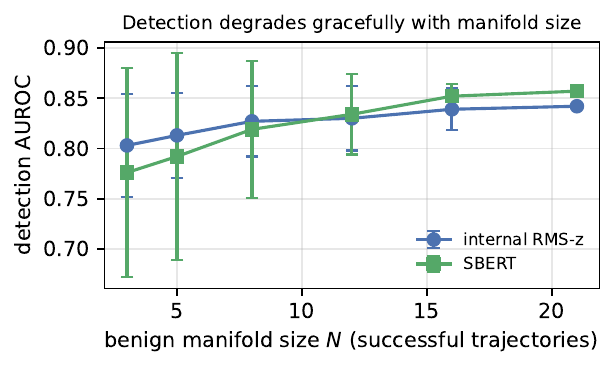}
\caption{Detection AUROC versus trusted-reference size for the internal and SBERT signals (mean $\pm$ std over $50$ resamples). Both degrade smoothly, and a few tens of successful trajectories suffice.}
\label{fig:bensize}
\end{figure}

Splitting the sweep by domain guards against a mismatch between them. The pooled sweep could hide a mismatch between domains, so the sweep is repeated within each domain alone. The reference set is resampled from that domain's own successes and detection is scored within the same domain, with $50$ resamples per $N$ (seed $0$, Table~\ref{tab:bensize-domain}). Airline detection stays between $0.81$ and $0.82$ (internal) from $N{=}3$ upward. Retail rises from $0.82$ at $N{=}2$ to $0.86$ at $N{=}5$ (SBERT), with the internal signal stable near $0.72$. The five-success retail reference set therefore behaves like the small-$N$ end of a smooth curve in its own domain, not like a side effect of pooling domains together.

\begin{table}[h]
\centering
\small
\resizebox{\columnwidth}{!}{%
\begin{tabular}{llccccc}
\toprule
Split & Signal & $N{=}2$ & $3$ & $5$ & $8$ & $16$ \\
\midrule
\multirow{2}{*}{Airline} & Internal & -- & $0.809$ & $0.824$ & $0.819$ & $0.818$ \\
 & SBERT & -- & $0.768$ & $0.765$ & $0.766$ & $0.772$ \\
\midrule
\multirow{2}{*}{Retail} & Internal & $0.706$ & $0.723$ & $0.715$ & -- & -- \\
 & SBERT & $0.822$ & $0.830$ & $0.860$ & -- & -- \\
\bottomrule
\end{tabular}}
\caption{Per-domain trusted-reference size sweep (8B, mean over $50$ resamples, each domain scored on itself). Each sweep stops at that domain's success count.}
\label{tab:bensize-domain}
\end{table}

The $\tau$-bench 8B runs have only $21$ successes, so the pooled sweep stops at $N{=}21$. On the easier GSM8K task, Qwen2.5-7B succeeds often enough ($63$ successes) to extend the sweep to $N{=}50$. Absolute AUROC is lower there, since a wrong arithmetic answer shifts the state less than a give-up. Both signals are already flat by $N{\approx}10$ and gain under $0.01$ AUROC out to $N{=}50$ (internal $0.678$ at $N{=}3$ to $0.702$ at $N{=}50$, SBERT $0.689$ to $0.723$), with shrinking variance. A few tens of successful trajectories again suffice to set up the reference statistics.

\subsection{Trusted-set contamination}
\label{app:contamination}

Trusted trajectories are picked from downstream outcomes, so the calibration set can include some failures by mistake. To test this effect, $k$ failing trajectories are added to the pooled 8B trusted set (21 successes). The internal reference statistics $(\mu,\sigma)$ and the SBERT trusted mean are re-estimated on that mixed set, and detection AUROC is recomputed on the remaining runs, averaged over $50$ resamples per rate (Table~\ref{tab:contamination}). The two signals react differently. The internal signal stays almost unchanged up to $30\%$ contamination (from $0.842$ to $0.849$), while the SBERT signal drops steadily (from $0.857$ to $0.753$).

The difference comes from how each signal is calibrated. The internal reference is estimated over \emph{steps}. Most steps in a failing trajectory still look normal, since an omission follows normal behaviour until the point where it stops. Those mostly normal steps barely move the per-dimension moments. The embedding reference uses one vector per \emph{trajectory}, so each contaminant carries its distinctive give-up text straight into the trusted mean. This is the second setting, after truly silent failures under stochastic decoding (Appendix~\ref{app:layer}), where the internal signal is the more stable of the two. In settings where clean trusted labels cannot be guaranteed, this gives a concrete reason to prefer the internal signal, or at least to include it alongside text-based detection.

\begin{table}[h]
\centering
\small
\resizebox{\columnwidth}{!}{%
\begin{tabular}{lccccc}
\toprule
Contamination & $0\%$ & $4.5\%$ & $8.7\%$ & $19.2\%$ & $30\%$ \\
\midrule
Internal & $0.842$ & $0.846$ & $0.844$ & $0.848$ & $0.849$ \\
SBERT    & $0.857$ & $0.852$ & $0.844$ & $0.816$ & $0.753$ \\
\bottomrule
\end{tabular}}
\caption{Detection AUROC with the trusted calibration set contaminated by unlabeled failing trajectories (pooled $8$B, mean over $50$ resamples).}
\label{tab:contamination}
\end{table}

\subsection{Stale trusted set under drift}
\label{app:drift}

The main runs use greedy decoding, so a fixed trusted reference reflects only that decoding setting. To mimic a changing agent, the 8B airline split is regenerated three times with temperature-$0.7$ sampling, giving $150$ changed trajectories ($41$ successful, $109$ failing). Table~\ref{tab:drift} scores this changed agent twice, once using the original trusted reference and once using a reference recalibrated on the changed trusted traffic. Using a stale trusted reference lowers AUROC, and recalibrating on the new trusted traffic restores the AUROC.

\begin{table}[h]
\centering
\small
\begin{tabular}{llc}
\toprule
Agent config & Reference     & AUROC \\
\midrule
Original (greedy)    & Original       & $0.818$ \\
Drifted (temp 0.7)   & Stale original & $0.793$ \\
Drifted (temp 0.7)   & Recalibrated   & $0.814$ \\
\bottomrule
\end{tabular}
\caption{Staleness under agent drift (8B airline, internal RMS-$z$).}
\label{tab:drift}
\end{table}

\subsection{Detection operating points}
\label{app:oppoint}

The main text fixes the flag threshold at a $10\%$ trusted false-positive rate (FPR). Table~\ref{tab:oppoint} reports precision, recall, and $F_1$ on the pooled 8B runs at trusted FPR targets of $5\%$, $10\%$, and $20\%$ for the internal anomaly, the SBERT text embedding, and the equal-weight fusion. Precision stays near $0.97$ to $0.99$ across these settings because failures dominate the pooled runs (144 of 165 trajectories, an $87\%$ failure rate that no real deployment matches). Recall rises as the FPR allowance grows. A deployment can therefore trade missed failures against false alarms by moving the trusted threshold alone, without failure labels.

Precision depends on how common failures are, so the values above overstate what a deployment sees. AUROC and the per-class rates (recall, FPR) do not depend on how common failures are, but precision does, and the values above overstate what a real system would see. At a deployment failure rate $\pi$, the expected precision implied by the reported rates is
\[
\frac{\mathrm{TPR}\,\pi}{\mathrm{TPR}\,\pi+\mathrm{FPR}(1-\pi)}.
\]
For the fused detector at the $10\%$ FPR point (recall $0.632$), this gives precision $0.25$ at $\pi{=}0.05$ and $0.06$ at $\pi{=}0.01$. Tightening to the $5\%$ FPR point (recall $0.576$) gives $0.38$ and $0.10$. At realistic failure rates most flags are therefore false alarms. The threshold needs to be chosen from the alert volume an operations team can handle, not from precision on this pooled set. This behaviour is shared by any detector for rare failures at these AUROC levels and is the reason the threshold is exposed as an explicit trusted-quantile parameter rather than fixed.

\begin{table}[h]
\centering
\small
\resizebox{\columnwidth}{!}{%
\begin{tabular}{llccc}
\toprule
Signal & FPR & Precision & Recall & $F_1$ \\
\midrule
\multirow{3}{*}{Internal RMS-z} & $5\%$  & $0.986$ & $0.472$ & $0.638$ \\
 & $10\%$ & $0.976$ & $0.556$ & $0.708$ \\
 & $20\%$ & $0.960$ & $0.674$ & $0.792$ \\
\midrule
\multirow{3}{*}{SBERT} & $5\%$  & $0.986$ & $0.493$ & $0.657$ \\
 & $10\%$ & $0.978$ & $0.625$ & $0.763$ \\
 & $20\%$ & $0.966$ & $0.785$ & $0.866$ \\
\midrule
\multirow{3}{*}{Fused (equal)} & $5\%$  & $0.988$ & $0.576$ & $0.728$ \\
 & $10\%$ & $0.978$ & $0.632$ & $0.768$ \\
 & $20\%$ & $0.962$ & $0.708$ & $0.816$ \\
\bottomrule
\end{tabular}}
\caption{Detection precision, recall, and $F_1$ at three benign-calibrated FPR operating points (pooled Llama-3.1-8B airline$+$retail). The threshold is set on benign traffic alone.}
\label{tab:oppoint}
\end{table}

\subsection{Supervised upper bound}
\label{app:skyline}

To measure what is lost by the label-free constraint, two supervised references are trained on the pooled 8B runs (144 failing, 21 successful trajectories) with gold success labels. Each supervised detector is scored by stratified five-fold cross-validation so every trajectory is scored out of fold, over five random seeds. The first detector is a logistic regression on the mean-pooled per-step activation (4096-dimensional), the supervised analogue of a success probe \citep{agentprobe2026}. The second is a logistic regression over the three per-trajectory signals of Section~\ref{sec:method} (internal $\max_t z_t$, tool error, SBERT trusted-distance), which matches the label-tuned fusion ruled out by Proposition~\ref{prop:labelfree}. The label-free signals used for comparison are still calibrated on the trusted set as in the main protocol.

\begin{table}[h]
\centering
\small
\begin{tabular}{lc}
\toprule
Detector & AUROC \\
\midrule
Label-free internal RMS-$z$ & $0.842$ \\
Label-free SBERT & $0.857$ \\
Label-free equal-weight fusion & $0.859$ \\
Supervised probe (mean-pool act.) & $0.793\pm0.028$ \\
Supervised fusion (three signals) & $0.862\pm0.008$ \\
\bottomrule
\end{tabular}
\caption{Supervised upper bound versus label-free detection (pooled 8B, supervised rows five-fold CV over five seeds).}
\label{tab:skyline}
\end{table}

Table~\ref{tab:skyline} shows that the upper bound is essentially flat. Label-tuned fusion weights gain only $+0.003$ over the equal-weight label-free fusion, well within one CV standard deviation. The supervised activation probe falls below the unsupervised distance because 21 positives are not enough to fit a 4096-dimensional boundary that holds up across folds. On this pooled set, enforcing the label-free constraint therefore costs no measurable detection power.

A dataset with many more labelled successes could change the result and make the supervised activation model look better than the label-free one. The supervised fusion also uses the same trusted calibration as its input signals, so the test only measures how labels help to set the weights, not how they might help to learn new features.

\subsection{Internal-vs-embedding disagreement}
\label{app:disagree}

Three sentence encoders are compared, \texttt{all-MiniLM-L6-v2} for SBERT, \texttt{all-mpnet-base-v2} for MPNet and \texttt{BAAI/bge-small-en-v1.5} for BGE, each with its default mean pooling, unit-normalised embeddings and the assistant text truncated to $4000$ characters. All three land in the same AUROC band as the internal signal and well above lexical TF-IDF on the pooled 8B runs (SBERT, MPNet $0.854$, BGE $0.814$). This makes it unlikely that the lack of separation between internal and text-based detection is an artefact of a single embedding family. At the $10\%$-FPR operating point on the pooled 8B runs, the internal and SBERT detectors agree on most failing trajectories (62 both flag, 36 both miss). They disagree on 46, with 18 flagged by the internal signal alone and 28 by the embedding alone, with 44 of the 46 in the retail split.

The two disagreement types are interpretable. The internal-only cases are well-formed tool calls with made-up arguments, for example a \texttt{modify\_pending\_order} on the placeholder order id \texttt{\#W0000000}. The surface text reads like a normal call, so the embedding stays near trusted while the activation is anomalous. The embedding-only cases are short hand-offs, for example a single-step \texttt{transfer\_to\_human\_agents}. The give-up text is distinctive, but the brief activation does not move far from the trusted reference. Neither type lines up with a single simple failure pattern. Both concentrate in the split with the smallest trusted set, which supports a story where calibration size, not representation family, drives the disagreement.

The internal–SBERT tie could in principle come from mixing two opposite types of failure. Internal state would win where the agent \emph{acted}. It issued a consequential write-tool call with wrong content, so the transcript reads as normal operation. The embedding would win where the agent \emph{gave up} and produced a distinctive hand-off text. Splitting the 144 pooled failures by this axis rules out that explanation (Table~\ref{tab:regime}). Given the 21 successes, this is a statement about power rather than about equivalence, and the separation appears only on the larger datasets of Table~\ref{tab:general}.

\begin{table}[h]
\centering
\small
\resizebox{\columnwidth}{!}{%
\begin{tabular}{lcccc}
\toprule
Failure regime & $n$ & internal & SBERT & DeLong $p$ \\
\midrule
Acted (write call, wrong) & $82$ & $0.902$ & $0.938$ & $0.39$ \\
Gave up (transfer / no write) & $62$ & $0.763$ & $0.750$ & $0.80$ \\
\bottomrule
\end{tabular}}
\caption{Detection AUROC by failure regime (pooled $8$B corpus, each regime's failures against all $21$ successes). The internal SBERT tie holds within both regimes.}
\label{tab:regime}
\end{table}

\section{Attribution audit details}
\subsection{Who\&When text-arm details}
\label{app:whoandwhen}
The Who\&When text-only arm tests whether label-free attribution still works when the gold step is not tied to a fixed position.

\begin{table}[h]
\centering
\small
\resizebox{\columnwidth}{!}{%
\begin{tabular}{llcc}
\toprule
Split & Method & step-acc [95\% CI] & agent \\
\midrule
\multirow{4}{*}{Hand-crafted} & Label-free & $0.172$ $[0.086,0.276]$ & $\mathbf{0.655}$ \\
 & GPT-4o judge & $0.172$ & $0.500$ \\
 & Always-last & $0.121$ $[0.034,0.207]$ & -- \\
 & Random & $0.042$ & -- \\
\midrule
\multirow{4}{*}{Algorithm-gen.} & Label-free & $0.127$ $[0.071,0.190]$ & $0.278$ \\
 & GPT-4o judge & $\mathbf{0.302}$ & $0.429$ \\
 & Always-first & $0.159$ $[0.095,0.222]$ & -- \\
 & Random & $0.120$ & -- \\
\bottomrule
\end{tabular}}
\caption{Who\&When attribution (text-only, no gold seen). Bootstrap CIs in brackets.}
\label{tab:whoandwhen}
\end{table}

The Who\&When benchmark \citep{whoandwhen2025} provides 58 hand-crafted and 126 algorithm-generated failed multi-agent traces with gold mistake-step and mistake-agent labels and no activations.
The label-free text signal combines error and retry markers, repetition or loop indicators from TF-IDF cosine similarity to the previous step, and task drift.
It predicts the decisive step as the argmax of this combined score and predicts the mistake agent as the agent acting at that step.
On hand-crafted traces, it reaches step accuracy $0.172$ and mistake-agent accuracy $0.655$, compared to a GPT-4o judge at $0.172$ and $0.500$ (Table~\ref{tab:whoandwhen}).
On algorithm-generated traces, the judge is stronger, with step accuracy $0.302$ versus $0.127$.

Because the label-free method and the always-last baseline are scored on the same traces, a paired test is more informative than comparing overlapping marginal intervals.
On the hand-crafted split the two predictors disagree on only three traces, and the label-free prediction is correct on all three (discordant counts $b{=}3$, $c{=}0$).
This gives a paired-bootstrap lift confidence interval of $[0.00,0.12]$ and a one-sided paired-bootstrap $p{=}0.044$, stable across seeds.
The label-free method is never wrong on a trace where always-last is right on this split.
With only three discordant pairs, an exact one-sided McNemar test still bottoms out at $p{=}0.125$ (two-sided $0.25$), so the advantage is consistent but not strongly powered.

This text-only arm carries no internal state.
The two tests disagree in strength, at $p{=}0.044$ paired and $p{=}0.25$ unpaired. The paired form conditions on the three discordant traces, while the unpaired form resamples all 58. The body quotes the unpaired result, since that is what the rest of the audit uses. Neither reading extends to TRAIL (Appendix~\ref{app:trail}), where the decisive error sits inside long real traces.
A non-positional gold is therefore necessary but not sufficient for label-free localization.

\subsection{Published attribution results under the correction}
\label{app:publishedlift}
Specialist models cannot be re-run, because their weights are not released. Their reported accuracies are still exact counts on the audited splits, under the same metric, so the
correction applies to them directly.

The baseline is treated as estimated rather than known. Trajectories are drawn with replacement and the
best content-blind rule is recomputed on each draw. The method itself is modelled as a binomial draw
at its reported rate. Six of the ten published results clear under that test (Table~\ref{tab:publishedlift}). Treating the baseline as a known fixed rate rather than an estimate would credit eight of the ten, so the stricter form is used here.

Without per-trajectory predictions the test cannot condition on which trajectories each method got
right, so it stays weaker than the paired test used elsewhere.

\begin{table}[h]
\centering
\small
\resizebox{\columnwidth}{!}{%
\begin{tabular}{llcccc}
\toprule
Method & split & acc & lift & lift $95\%$ CI & clears \\
\midrule
\multicolumn{6}{l}{\emph{algorithm-generated, always-first $0.159$, $n{=}126$}} \\
RAFFLES            & alg & $0.516$ & $+0.357$ & $[+0.25,+0.46]$ & yes \\
AgenTracer-8B (w/G)   & alg & $0.429$ & $+0.270$ & $[+0.16,+0.37]$ & yes \\
AgenTracer-8B (w/oG)  & alg & $0.373$ & $+0.214$ & $[+0.11,+0.32]$ & yes \\
CORRECT            & alg & $0.381$ & $+0.222$ & $[+0.12,+0.33]$ & yes \\
Who\&When best     & alg & $0.254$ & $+0.095$ & $[-0.01,+0.19]$ & no \\
\midrule
\multicolumn{6}{l}{\emph{hand-crafted, always-last $0.121$, $n{=}58$}} \\
CHIEF              & hand & $0.293$ & $+0.172$ & $[+0.02,+0.31]$ & yes \\
RAFFLES            & hand & $0.276$ & $+0.155$ & $[+0.02,+0.29]$ & yes \\
AgenTracer-8B      & hand & $0.207$ & $+0.086$ & $[-0.05,+0.21]$ & no \\
CORRECT            & hand & $0.172$ & $+0.051$ & $[-0.07,+0.17]$ & no \\
Who\&When best     & hand & $0.086$ & $-0.034$ & $[-0.14,+0.07]$ & no \\
\bottomrule
\end{tabular}}
\caption{Published Who\&When results under the positional correction. Accuracies are exact integer counts on the full split, transcribed from the primary sources. Intervals resample trajectories, recomputing the best content-blind rule on each draw, so the baseline's own uncertainty is included. A fixed-baseline binomial test would credit two more rows.}
\label{tab:publishedlift}
\end{table}

\subsection{A published method under its own implementation}
\label{app:official}
The benchmark's released implementation was run unmodified with Qwen2.5-7B-Instruct on both splits,
changing no prompt and no parser. Its code always supplies the task answer, so these are the easier
with-ground-truth numbers.

Under the all-at-once protocol the accuracy lands close to the reported value on both splits, and
under the correction it clears neither split (Table~\ref{tab:official}).

The step-by-step protocol does not reproduce, because its loop records a step only when the reply opens
with a fixed affirmative string. The 7B model often answers off-format or flags nothing, which the code reads as no error. That is a capacity gap rather than a scoring artefact, so no conclusion is
drawn from those rows.

\begin{table}[h]
\centering
\small
\resizebox{\columnwidth}{!}{%
\begin{tabular}{llcccc}
\toprule
Protocol & Split & acc & reported & lift & $p$ \\
\midrule
all-at-once   & alg-gen & $0.151$ & $0.125$ & $-0.008$ & $0.61$ \\
all-at-once   & hand    & $0.052$ & $0.053$ & $-0.069$ & $0.95$ \\
\midrule
step-by-step\textsuperscript{$\ast$} & alg-gen & $0.040$ & $0.255$ & $-0.119$ & $1.00$ \\
step-by-step\textsuperscript{$\ast$} & hand    & $0.035$ & $0.070$ & $-0.086$ & $0.98$ \\
\bottomrule
\end{tabular}}
\caption{A published method run under its own released code. \emph{reported} is the accuracy its paper gives for that cell, \emph{lift} the gain over the matched positional rule. $^{\ast}$this protocol did not reproduce its reported accuracy here, so these two rows support no conclusion about the method.}
\label{tab:official}
\end{table}

\subsection{Two attribution models under the audit}
\label{app:raffles}
An open reimplementation of a published specialist architecture is run under the audit, written from
its description with a Qwen2.5-32B-Instruct backbone because the original weights are unreleased. A
current reasoning model, Qwen3-8B, is run alongside it under the same protocol as the GPT-4o judge.

Under the correction neither model clears its baseline (Table~\ref{tab:twomodels}). The
reimplementation clears nominally on the algorithm-generated split, but falls outside the audit's
corrected threshold of $p\le0.025$.

The reimplementation is only partial, since its prompts and decoding were reconstructed rather than
recovered. It also scores well below the accuracy the original reports. It therefore stands for the
architecture rather than for that system. The reasoning model failed to emit a usable
index on some traces, which are scored as wrong. Dropping those traces
instead leaves the hand-crafted verdict unchanged and makes the algorithm-generated lift nominally
positive, still short of the threshold.

\begin{table}[h]
\centering
\small
\resizebox{\columnwidth}{!}{%
\begin{tabular}{llccccc}
\toprule
Model & Split & acc & base & lift & $95\%$ CI & $p$ \\
\midrule
\multirow{2}{*}{Reimplementation} & alg-gen & $0.254$ & $0.159$ & $+0.095$ & $[0.01,0.18]$ & $0.032$ \\
 & hand & $0.155$ & $0.121$ & $+0.034$ & $[-0.09,0.16]$ & $0.38$ \\
\midrule
\multirow{2}{*}{Reasoning model} & alg-gen & $0.214$ & $0.159$ & $+0.056$ & $[-0.03,0.14]$ & $0.148$ \\
 & hand & $0.138$ & $0.121$ & $+0.017$ & $[-0.09,0.12]$ & $0.495$ \\
\bottomrule
\end{tabular}}
\caption{Two attribution models scored against their matched positional baselines. \emph{base} is the best content-blind rule on that split. No row survives the audit's corrected threshold of $p\le0.025$.}
\label{tab:twomodels}
\end{table}

\subsection{Field-wide audit, per split}
\label{app:field}
Table~\ref{tab:fieldfull} gives the per-split values behind the field audit of Section~\ref{sec:exp-field}, and Figure~\ref{fig:liftsummary} plots the lift of the best learned localizer across the nine settings that carry one.
For each split the three content-blind rules are scored on that split's own gold positions. The best is then re-selected on a random half and scored on the other half, over $2000$ splits, which removes selection bias. A run of length $n_\tau$ gives a random guess a chance of $1/n_\tau$, so the random baseline is $\mathbb{E}_\tau[1/n_\tau]$.
Mean selection bias across the 13 splits is $0.008$. In-sample and held-out ceilings are therefore close everywhere, so the correction matters for the principle rather than for these particular numbers.
ProcessBench contributes only its erroneous samples, since a correct solution has no first error to localize. Its four subsets are treated as separate splits, because they differ in solution length and generator.

\begin{figure}[tb]
\centering
\includegraphics[width=0.82\linewidth]{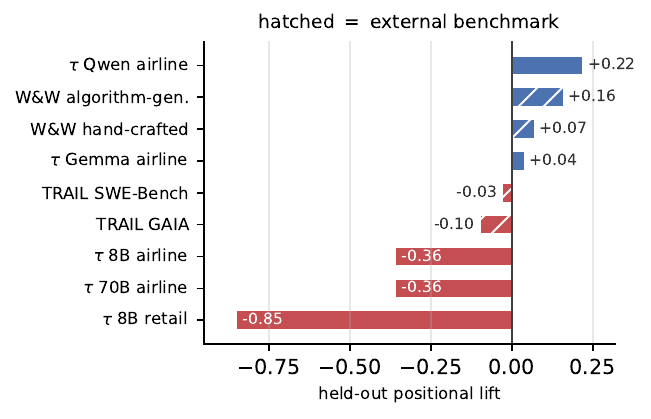}
\caption{Positional lift of the best learned localizer across nine attribution settings, sorted. Hatching marks external benchmarks.}
\label{fig:liftsummary}
\end{figure}

\begin{table}[h]
\centering
\small
\resizebox{\columnwidth}{!}{%
\begin{tabular}{llrrrrr}
\toprule
Benchmark & split & $n$ & steps & pos. & ceil. & rand. \\
\midrule
\multirow{5}{*}{SILENT-Bench}
 & $\tau$ 70B airline   & $39$  & $8.6$  & $0.84$ & $0.666$ & $0.138$ \\
 & $\tau$ 8B airline    & $34$  & $5.5$  & $0.98$ & $0.884$ & $0.502$ \\
 & $\tau$ 8B retail     & $110$ & $4.8$  & $0.99$ & $0.964$ & $0.309$ \\
 & $\tau$ Qwen airline  & $47$  & $9.4$  & $0.66$ & $0.106$ & $0.111$ \\
 & $\tau$ Gemma airline & $27$  & $9.6$  & $0.82$ & $0.400$ & $0.106$ \\
\midrule
\multirow{2}{*}{Who\&When}
 & hand-crafted         & $58$  & $51.6$ & $0.46$ & $0.106$ & $0.042$ \\
 & algorithm-gen        & $126$ & $8.7$  & $0.37$ & $0.141$ & $0.120$ \\
\midrule
\multirow{2}{*}{TRAIL}
 & GAIA                 & $113$ & $31.2$ & $0.53$ & $0.134$ & $0.057$ \\
 & SWE-Bench            & $30$  & $34.2$ & $0.26$ & $0.000$ & $0.034$ \\
\midrule
\multirow{4}{*}{ProcessBench}
 & GSM8K                & $207$ & $5.3$  & $0.42$ & $0.191$ & $0.208$ \\
 & MATH                 & $594$ & $6.8$  & $0.39$ & $0.227$ & $0.181$ \\
 & OlympiadBench        & $661$ & $8.9$  & $0.43$ & $0.170$ & $0.136$ \\
 & OmniMATH             & $759$ & $8.6$  & $0.39$ & $0.157$ & $0.138$ \\
\bottomrule
\end{tabular}}
\caption{Per-split positional audit. \emph{ceil.} held-out positional ceiling, \emph{rand.} random baseline, \emph{pos.} mean gold relative position.}
\label{tab:fieldfull}
\end{table}

\subsection{Method arms on ProcessBench}
\label{app:fieldmethods}
A discriminant check asks whether the measure credits genuine skill rather than only penalising methods.

To check that the standard credits genuine skill rather than only penalizing methods, two localizers are run on ProcessBench.
In this benchmark the positional ceiling sits at the random baseline and cannot carry a method on its own.
Two judges from different model families, Qwen2.5-7B and Mistral-Small-24B, are shown the numbered solution steps and asked for the index of the first wrong step.
They are prompted in one-based indexing and their answers are converted on parse.
For each subset, 200 erroneous instances are sampled, and neither judge produces unparseable replies.
The label-free text localizer is the same one used on Who\&When.

Both judges score above the ceiling on every subset, while the label-free heuristic does not (Table~\ref{tab:fieldmethods}).
Their mean predicted position is $0.41$ to $0.55$, near the middle of the sequence, so the gain is not itself a positional strategy.

\begin{table}[h]
\centering
\small
\resizebox{\columnwidth}{!}{%
\begin{tabular}{lccccccc}
\toprule
 & & \multicolumn{2}{c}{Qwen-7B} & \multicolumn{2}{c}{Mistral-24B} & \multicolumn{2}{c}{label-free} \\
\cmidrule(lr){3-4}\cmidrule(lr){5-6}\cmidrule(lr){7-8}
Subset & ceil. & acc & lift & acc & lift & acc & lift \\
\midrule
GSM8K         & $0.193$ & $0.520$ & $+0.327$ & $0.485$ & $+0.292$ & $0.220$ & $+0.027$ \\
MATH          & $0.189$ & $0.380$ & $+0.191$ & $0.330$ & $+0.141$ & $0.165$ & $-0.024$ \\
OlympiadBench & $0.184$ & $0.320$ & $+0.136$ & $0.275$ & $+0.091$ & $0.165$ & $-0.019$ \\
OmniMATH      & $0.225$ & $0.330$ & $+0.105$ & $0.320$ & $+0.095$ & $0.175$ & $-0.050$ \\
\bottomrule
\end{tabular}}
\caption{Discriminant check on ProcessBench ($200$ erroneous instances per subset, no unparseable replies). Both judge families earn positive lift on every subset, while a content-blind-in-practice text heuristic does not.}
\label{tab:fieldmethods}
\end{table}

\subsection{Audit restricted to well-posed failures}
\label{app:wellposed}
Restricting the audit to failures that have a replay-verifiable decisive step isolates the effect of the labelling rule.

Recoverability is checked at every prefix of a failing run. The question at each prefix is whether the
remaining gold write actions can still reproduce the target state. A run has a state-verified decisive
step exactly when the answer flips from yes to no. The check reuses the
environment's own transition and hashing functions, and the user simulator is stubbed, since nothing
here depends on the user's text.

Of the $596$ failures in the regenerated corpus, $66$ flip and $530$ do not
(Table~\ref{tab:wellposed}). The greedy corpus gives $39$ of $257$, so the fraction is stable across
vintages. It is strongly domain-dependent, at $36.4\%$ on airline against $2.3\%$ on retail. Retail is
also the split with the most saturated positional ceiling, at $0.964$. A split's ceiling is largely a
readout of how ill-posed it is.

The positional rules are then scored on the well-posed subset alone, with the state-verified step as gold. That gives a pooled held-out ceiling of $0.442$ $[0.289, 0.605]$ against a random baseline of $0.425$. The
interval contains the random baseline, so the ceiling collapses to something indistinguishable from
chance. Under the conventional give-up gold the same splits reach ceilings up to $0.964$. The gold's
mean relative position moves from $0.84$--$1.00$ to $0.42$, which is what the collapse reflects.

The subset is small at $n{=}38$, so it supports the claim about the ceiling rather than a ranking of
localizers on it.

\begin{table}[h]
\centering
\small
\resizebox{\columnwidth}{!}{%
\begin{tabular}{lccc}
\toprule
Corpus / domain & failures & state-verified & \% \\
\midrule
Regenerated, airline & $154$ & $56$ & $36.4$ \\
Regenerated, retail  & $442$ & $10$ & $2.3$ \\
\midrule
Greedy pull, all splits & $257$ & $39$ & $15.2$ \\
\midrule
Regenerated, total   & $596$ & $66$ & $11.1$ \\
\bottomrule
\end{tabular}}
\caption{Failures admitting a state-verified decisive step, by corpus and domain.}
\label{tab:wellposed}
\end{table}

\subsection{Recoverability in two further environments}
\label{app:tau2}
\label{app:swebench}
Two further environments test whether the recoverability fraction is specific to one reward.

$\tau^2$-bench \citep{tau2bench2025} scores with agent-database and user-database checks, environment assertions and
required-action checks rather than a single hash, and ships full simulation results from frontier
agents. The same recoverability test is expressed in its own machinery. Replay a truncated message
history, apply the task's golden actions, and compare database hashes against a gold environment,
cutting only at complete turn boundaries.

A positive control gates the probe, since at full trajectory length it must reproduce the benchmark's
own scoring. Which comparison passes that control differs by domain (Table~\ref{tab:environments}),
and airline fails its control, so it is not reported.

The third setting is a coding agent on SWE-bench Verified \citep{swebench2024,swebenchverified2024}, where the state is a repository and success is a deterministic
test suite. Each trajectory is replayed prefix by prefix inside the benchmark's own container image,
applying the reference repair after every cut and re-running the tests.

The resulting fractions span the range reported in the main text, from $7.6\%$ on the coding agent to
$87.4\%$ on telecom, against $11.1\%$ on the corpus used elsewhere here. The two $\tau^2$ domains
also differ from each other (Fisher $p<0.001$), so the fraction varies within a benchmark as well as
between benchmarks.

Two readings of the coding-agent repair give very different answers. The strict reading above
requires the repaired workspace to pass. Requiring only that the reference patch still apply, which is
what a patch-level tool would see, leaves $52.4\%$ unrecoverable. A swing of that size from the
definition alone is a further reason to treat this fraction as measured rather than assumed.

The strict reading also supplies the audit row of Table~\ref{tab:attribution}, localizing 120 of the
198 replayed trajectories with the decisive step mid-trajectory at mean relative position $0.61$.
That is the property the $\tau$-bench splits lack, and even so a content-blind rule reaches
$3.3\times$ random and the label-free localizer does not clear it.

\begin{table}[h]
\centering
\small
\resizebox{\columnwidth}{!}{%
\begin{tabular}{lllcc}
\toprule
Benchmark & Domain & Control & $n$ & \% recov. \\
\midrule
$\tau^2$-bench & retail  & $116/120$ & $60$  & $63.3$ $[0.51,0.74]$ \\
$\tau^2$-bench & telecom & $100/100$ & $199$ & $87.4$ $[0.82,0.91]$ \\
$\tau^2$-bench & airline & $80/100$  & --    & not reported \\
\midrule
SWE-bench & code & $198/203$ & $198$ & $7.6$ $[0.04,0.12]$ \\
\bottomrule
\end{tabular}}
\caption{Recoverability in two further environments. \emph{Control} is the share of full-length trajectories on which the probe reproduces the benchmark's own scoring. A domain is reported only if its control passes.}
\label{tab:environments}
\end{table}

\subsection{TRAIL external attribution audit}
\label{app:trail}
TRAIL \citep{trail2025} supplies 143 real agent traces with human-annotated error spans, split into
information-retrieval and software-engineering sources.

The gold step here is non-positional, with mean relative position $0.53$ and $0.26$, and the traces
are long at about 32 spans. The positional baselines are therefore weak, at $0.10$ and $0.03$ (Table~\ref{tab:trail}).

Both methods still fall below those baselines. A GPT-4o judge given each step's name, status and
attributes lands within two steps of the gold on $18\%$ of traces, but never hits it exactly. The label-free text method that
comes closest to a gain on Who\&When falls clearly below the baseline here, at pooled lift $-0.077$
with paired interval $[-0.126, -0.035]$.

A non-positional gold makes step finding possible but not easy, and the marginal Who\&When result
does not carry over to long real traces.

\begin{table}[h]
\centering
\small
\resizebox{\columnwidth}{!}{%
\begin{tabular}{lccccc}
\toprule
TRAIL split & pos. & best pos. & learned & lift & PSS \\
\midrule
GAIA (retr.) & $.53$ & mid $0.10$ & $0.00$ & $-0.10$ & $-0.11$ \\
SWE (code)   & $.26$ & mid $0.03$ & $0.00$ & $-0.03$ & $-0.03$ \\
\bottomrule
\end{tabular}}
\caption{TRAIL positional-lift audit ($143$ real agent traces). A GPT-4o judge fails to clear the weak positional baselines despite a non-positional gold.}
\label{tab:trail}
\end{table}

\subsection{GPT-4o judge predicted positions}
\label{app:judgepos}
The give-up gold on the 70B airline split is late, at mean relative position $0.87$.
Table~\ref{tab:judgepos} splits the 39 failing traces by length and reports where the GPT-4o judge
places the decisive step, and Figure~\ref{fig:judgepos} shows the same distribution. On traces of at least six steps it picks one of the first two steps in 18
of 32 cases, and hits the late gold in only 6. On traces of at most five steps, where early and late
nearly coincide, it matches all 7. The judge is looking for an early root cause rather than the late
give-up point, and inspecting every step does not save it from losing to the always-last rule.

The prompt asks for the decisive step but never says that an omission's gold is the give-up point.
Part of that gap is therefore a mismatch in definitions rather than a failure to read the run. Table~\ref{tab:promptablation}
tests this within one model, scoring an open judge on each airline split under the original prompt and
under one that states the convention. That change moves the judge's mean predicted position from early to late ($0.14\to0.54$ at 8B,
$0.28\to0.73$ at 70B). Accuracy roughly doubles and no trace gets worse, yet the judge still trails
always-last on both splits.

An open judge is used for that comparison so the prompt is the only thing that changes. The GPT-4o
judge keeps the original prompt in the main text, so its scores stay comparable to prior judge
baselines.

\begin{figure}[h]
\centering
\includegraphics[width=0.92\linewidth]{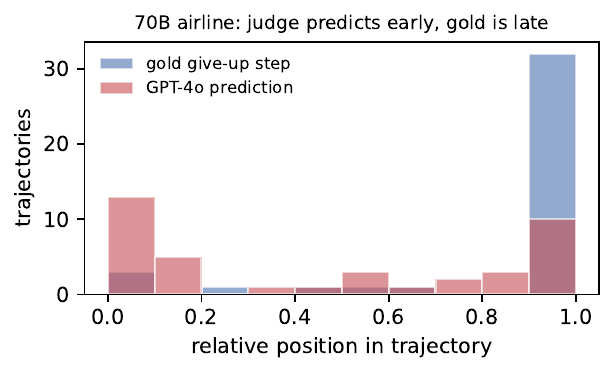}
\caption{Relative position of the gold give-up step and the GPT-4o judge's prediction, 70B airline split.}
\label{fig:judgepos}
\end{figure}

\begin{table}[h]
\centering
\small
\resizebox{\columnwidth}{!}{%
\begin{tabular}{lccc}
\toprule
Trace length & $n$ & pred.\ in first two & exact match \\
\midrule
Long ($\geq6$ steps)  & $32$ & $18$ ($56\%$) & $6$ ($19\%$) \\
Short ($\leq5$ steps) & $7$  & $1$           & $7$ ($100\%$) \\
\bottomrule
\end{tabular}}
\caption{GPT-4o judge predicted decisive-step positions on the 70B airline split, by trace length (gold give-up step is late, mean relative position $0.87$).}
\label{tab:judgepos}
\end{table}

\begin{table}[h]
\centering
\small
\resizebox{\columnwidth}{!}{%
\begin{tabular}{lccccc}
\toprule
Split & prompt & acc & pred.\ pos. & last & lift \\
\midrule
\multirow{2}{*}{8B airline} & naive & $0.41$ & $0.14$ & \multirow{2}{*}{$0.88$} & $-0.47$ \\
 & aware & $0.85$ & $0.54$ & & $-0.03$ \\
\midrule
\multirow{2}{*}{70B airline} & naive & $0.05$ & $0.28$ & \multirow{2}{*}{$0.67$} & $-0.62$ \\
 & aware & $0.31$ & $0.73$ & & $-0.36$ \\
\bottomrule
\end{tabular}}
\caption{Give-up-aware prompt ablation, open judge (Mistral-Small-24B). \emph{pred.\ pos.} mean predicted relative position, \emph{last} always-last accuracy.}
\label{tab:promptablation}
\end{table}

\section{Representation analysis}
\subsection{Geometry of the internal signal}
\label{app:geometry}
Two checks ask whether the failure signal sits in one direction or in many.

The first restricts the standardized distance to the top-$k$ trusted principal directions, and
separately to the complementary tail. The separation lives mostly outside the main trusted
directions, with the full distance at $0.842$ on the pooled agent corpus.

The second tests an explicit one-dimensional detector, the maximum projection onto the top trusted
direction, against the full-space distance. The full distance wins on four of five datasets with
recorded activations. GSM8K with Qwen is the single exception, where one direction suffices.

A sparse dictionary gives the same answer. Table~\ref{tab:sae} reports the main dictionaries, and the
verdict holds across eighteen of them, spanning six datasets, three model families, five widths and
three sparsity levels. Dictionaries are fitted on standardized trusted steps only, so they are
label-free, and labels enter only when scoring which latents separate. Sparsity is architectural
rather than penalised, keeping the $32$ largest pre-activations per step, because an $L_1$ penalty
did not control sparsity reliably at this data scale. Widths run from $256$ to $28672$ latents.

Because the firing criterion is a threshold, the best single-latent AUROC is also reported, which
needs none. It reaches only $0.38$ to $0.68$, never matching the full distance. Control dictionaries fitted on trusted and failing steps together, and on failing steps alone, still
produce no failure-specific latent. The best of an equal budget of random directions separates about
as well as the best learned latent.

The signal is therefore spread over many weakly shifting directions rather than stored in one unit.

\subsection{The label-free property: statement and proof}
\label{app:labelfree}
\begin{definition}[Label-free monitor]
\label{def:labelfree}
A monitor with parameters $\theta$ and operating threshold $\theta_{\text{thr}}$ fit by an estimator $\hat E$ is \emph{label-free} if $\hat E$ is measurable with respect to the trusted successful trajectories $\mathcal{B}$ alone, with no failure label $y$ entering any computation up to and including threshold-setting:
\[
\begin{aligned}
\theta&=\hat E(\mathcal{B}), \qquad \hat E \perp y,\\
\theta_{\text{thr}}&=Q_{1-\alpha}(\{D(b):b\in\mathcal{B}\}).
\end{aligned}
\]
\end{definition}

\begin{proposition}
\label{prop:labelfree}
Equal-weight $z$-score fusion (Eq.~\ref{eq:detect}) is label-free, and any signal weighting fit to a label-dependent objective is not.
\end{proposition}

\begin{proof}[Proof of Proposition~\ref{prop:labelfree}]
Every fitted quantity in $D$, namely $\bar s_{\text{ben}}$, $\hat\sigma_{\text{ben}}$, and the trusted quantile $\theta_{\text{thr}}$, is a function of $\mathcal{B}$ only, and the weights are fixed constants, so $\hat E\perp y$ and $D$ is label-free by Definition~\ref{def:labelfree}.
Conversely, any weighting $w^\star=\arg\max_w \mathrm{Obj}(w;\{(D_w(\tau_i),y_i)\})$ that maximizes a failure-labeled objective (AUROC, cross-entropy, or $F_1$ against $y$) is by construction a measurable function of the labels $y_i$. Even a single gradient step on such a loss makes $\theta$ depend on $y$, which breaks the definition.
\end{proof}

\subsection{Positional skill score: axioms and proofs}
\label{app:pss}
A step-attribution measure $M$ is required to satisfy:
\textbf{(A1) prior-insensitivity}, $M(f)\le0$ for any $f$ in the reference class $\mathcal{P}$;
\textbf{(A2) monotonicity}, $M(f)\ge M(f')$ if $f$ is correct on a superset of the trajectories $f'$ is;
\textbf{(A3) content-information soundness}, $M(f)>0$ only if $f$ uses information beyond the positions in $\mathcal{P}$;
\textbf{(A4) chance calibration}, $M$ reduces to accuracy-minus-chance under a uniform content-independent gold.

A3 is one-directional by design, so a positive score certifies that content was used. A non-positive
score does not certify the reverse, since a localizer can read content and still fail to clear a
strong positional rule.
A negative lift therefore means no skill was demonstrated against this reference, not that the method
ignores content.
Both guarantees are relative to the reference class $\mathcal{P}$. A content-blind rule outside it, such
as one that always names the third step, could in principle exceed $A^\star_{\text{pos}}$. A richer
reference class therefore gives a stronger guarantee at the cost of a higher ceiling.

Two larger content-blind families are scored on Who\&When
and ProcessBench under the same half-split held-out protocol as the main audit (Table~\ref{tab:bestindex}).
The first picks the single best fixed step number $k$ and always predicts step $\min(k, n-1)$. The second maps each run length to the gold step that is most common at that length in the fitting
half. That is the strongest rule the gold positions alone can support.

Both raise the reference on every split. The sharpest rise is on Who\&When algorithm-generated, where
always naming the second step scores $0.260$ held-out against $0.142$ for the three-rule family. Under
that harder reference the label-free localizer's lift on that split falls from $+0.14$ to $+0.04$, and
the benchmark's own best published result ($0.254$) drops below the content-blind bar. The specialist
systems that clear by a wide margin still clear, at $0.516$ for RAFFLES, $0.429$ for AgenTracer and
$0.381$ for CORRECT. On ProcessBench the reference rises by one to eight points, and every judge that
cleared keeps a positive margin, since those judges sit between $0.32$ and $0.52$. The permutation
tests are not re-run against this larger family.

The three-rule family stays the reference in the main text because it is what published methods
implicitly compete against. The strength of the A1 to A3 guarantee is nonetheless exactly the size of
that family, and Table~\ref{tab:bestindex} shows how far the bar moves when the family grows.

\begin{table}[h]
\centering
\small
\resizebox{\columnwidth}{!}{%
\begin{tabular}{lrcccc}
\toprule
Split & $n$ & 3-rule & fixed idx & length map & best $k$ \\
\midrule
W\&W hand-crafted  & $58$  & $0.104$ & $0.134$ & $0.116$ & $8$ \\
W\&W algorithm-gen & $126$ & $0.142$ & $0.260$ & $0.200$ & $1$ \\
PB gsm8k           & $207$ & $0.191$ & $0.276$ & $0.255$ & $1$ \\
PB math            & $594$ & $0.227$ & $0.240$ & $0.261$ & $1$ \\
PB olympiadbench   & $661$ & $0.170$ & $0.221$ & $0.203$ & $2$ \\
PB omnimath        & $759$ & $0.157$ & $0.238$ & $0.241$ & $1$ \\
\bottomrule
\end{tabular}}
\caption{Larger content-blind reference families, half-split held-out accuracy averaged over $2000$ splits. \emph{3-rule} is the paper's family $\mathcal{P}$, \emph{fixed idx} the best absolute index, \emph{length map} the length-conditional table, \emph{best $k$} the in-sample best fixed index (zero-indexed).}
\label{tab:bestindex}
\end{table}

\begin{proposition}
\label{prop:pss-axioms}
Positional lift $L(f)=\mathrm{Acc}(f)-A^\star_{\text{pos}}$ satisfies \textnormal{A1--A4}, whereas raw accuracy violates \textnormal{A1} and \textnormal{A3}.
\end{proposition}

\begin{proof}
(A1) For content-blind $f\in\mathcal{P}$, $\mathrm{Acc}(f)\le A^\star_{\text{pos}}$ by maximality, so $L(f)\le0$.
(A2) $A^\star_{\text{pos}}$ is independent of $f$ and $\mathrm{Acc}$ is monotone under prediction-superset, hence so is $L$.
(A3) $L(f)>0\implies\mathrm{Acc}(f)>A^\star_{\text{pos}}$, so $f$ beats every rule in $\mathcal{P}$, which requires information those rules do not have.
The converse does not hold and is not claimed. A content-reading $f$ with $\mathrm{Acc}(f)\le A^\star_{\text{pos}}$ scores $L(f)\le0$.
(A4) is Proposition~\ref{prop:pss-unbiased}.
Raw accuracy assigns the always-last rule $0.67$--$0.96$, violating A1, and would credit a method that never beats position, violating A3.
\end{proof}

\begin{proposition}[Unbiasedness under uniform gold]
\label{prop:pss-unbiased}
If $g\sim\mathrm{Uniform}\{1,\dots,n\}$ independently of content, then $A^\star_{\text{pos}}=\mathbb{E}_\tau[1/n_\tau]$ (the random baseline) and $L(f)=\mathrm{Acc}(f)-\mathbb{E}_\tau[1/n_\tau]$, an unbiased estimator of content-information gain over chance.
\end{proposition}

\begin{proof}
For a single predicted position $p=\mathrm{round}(\rho n)$, uniform $g$ gives $P(g{=}p)=1/n$.
Taking expectation over $\tau$ gives $\mathrm{Acc}(\pi_\rho)=\mathbb{E}[1/n_\tau]$ for every $\rho$, so the maximum over $\mathcal{P}$ is the same constant, equal to chance.
\end{proof}

\noindent
The normalized $\mathrm{PSS}=L/(1-A^\star_{\text{pos}})$ is the skill-score form of $L$ and is the quantity maximized in expectation, with the reference held fixed, by reporting the true $\arg\max_t P(g{=}t\mid\tau)$.
It is therefore proper in the skill-score sense.
PSS is unstable as $A^\star_{\text{pos}}\to1$ (the retail case), in which small lift maps to a large magnitude.
It is most informative on benchmarks whose positional baseline is not saturated.
A floored variant removes the blow-up while preserving the ordering. Replacing the denominator with
$\max(1-A^\star_{\text{pos}},\epsilon)$ leaves A1 to A4 intact, since it rescales a positive quantity
and does not change any sign. At $\epsilon=0.10$ the worst score over the 30 audited pairs moves from $-23.3$ to $-8.5$. The rank
correlation with the unfloored score is $1.00$, both over all pairs and over the subset whose
denominator is already well behaved. The two forms therefore agree wherever PSS
is interpretable, and only the saturated splits change. Table~\ref{tab:pssfull} gives both forms
for the splits of Table~\ref{tab:attribution}.

\begin{table}[h]
\centering
\small
\resizebox{\columnwidth}{!}{%
\begin{tabular}{lccc}
\toprule
Benchmark / split & PSS & $95\%$ CI & floored ($\epsilon{=}0.10$) \\
\midrule
$\tau$ 8B airline    & $-3.0$  & $[-14.0, -1.1]$ & $-3.0$ \\
$\tau$ 8B retail     & $-23.3$ & $[-97.0, -10.9]$ & $-8.5$ \\
$\tau$ 70B airline   & $-1.1$  & $[-2.5, -0.4]$ & $-1.1$ \\
$\tau$ Qwen airline  & $+0.20$ & -- & $+0.20$ \\
$\tau$ Gemma airline & $+0.06$ & $[0.00, 0.20]$ & $+0.06$ \\
W\&W hand-crafted    & $+0.06$ & $[0.00, 0.14]$ & $+0.06$ \\
W\&W algorithm-gen   & $+0.17$ & -- & $+0.17$ \\
SWE-bench Verified   & $-0.03$ & $[-0.08, 0.03]$ & $-0.03$ \\
\bottomrule
\end{tabular}}
\caption{Normalised skill score for the Table~\ref{tab:attribution} splits, unfloored and floored. Only the saturated retail split changes materially, and every sign is preserved.}
\label{tab:pssfull}
\end{table}

The reference itself carries selection bias, because $A^\star_{\text{pos}}$ is a maximum over three rules evaluated on the same split that is being scored. It is therefore an upward-biased estimate of the best population positional rule, and lift is biased downward to match.
The bias is largest on small splits, which is where the audit's negative results sit, so it has to be measured rather than assumed away.
The rule is instead selected on a random half of each split and scored on the other half, averaged over $2000$ splits. That gives a mean bias of $0.011$ and a maximum of $0.046$ across the seven splits of Table~\ref{tab:attribution}.
No split changes the sign of its lift under the correction, and the largest corrected lift is $0.216$ against $0.170$ uncorrected.
The audit's conclusions are therefore not an artifact of selecting the reference on the test split.
Reporting lift against an always-last rule specified in advance, which requires no selection at all, changes no conclusion either.

\subsection{Reference sample complexity}
\label{app:samplecomplexity}
The trusted reference is the per-dimension mean $\mu\in\mathbb{R}^d$ and standard deviation $\sigma\in\mathbb{R}^d$ estimated from $N$ successful trajectories. Write $\hat\mu_N,\hat\sigma_N$ for the estimates, $\hat z_N$ for the resulting per-step score, and $z_\infty$ for the population score.
The argument rests on two assumptions.
\textbf{(B1)} Each standardized trusted coordinate $u_j=(h_j-\mu_j)/\sigma_j$ is sub-Gaussian with proxy variance at most $1$, meaning its tails fall off at least as fast as a Gaussian's.
Final-layer residual streams are known to carry a few very large-magnitude coordinates, so this assumption is checked rather than asserted.
Across the five datasets with recorded activations, the median per-coordinate excess kurtosis is between $-0.64$ and $0.08$, at or below the Gaussian value. The share of coordinates above $3$ runs from $0.4\%$ to $11.7\%$ (Table~\ref{tab:kurtosis}).
Per-coordinate standardization therefore removes most of the outlier-dimension effect, and the heavy-tailed coordinates that remain are a small minority.
\textbf{(B2)} The population score has a density bounded by $B_\rho$ near the decision region.

\begin{table}[h]
\centering
\small
\resizebox{\columnwidth}{!}{%
\begin{tabular}{lccc}
\toprule
Corpus & median & p$95$ & share $>3$ \\
\midrule
GSM8K Qwen-7B      & $0.04$  & $2.75$ & $4.1\%$ \\
FinQA Qwen-7B      & $0.08$  & $5.07$ & $11.7\%$ \\
$\tau$-bench 8B    & $-0.64$ & $0.90$ & $0.4\%$ \\
FinQA Llama-8B     & $-0.16$ & $1.46$ & $1.0\%$ \\
GSM8K Llama-8B     & $-0.19$ & $2.43$ & $3.4\%$ \\
\bottomrule
\end{tabular}}
\caption{Per-coordinate excess kurtosis of the standardized trusted activations, testing assumption (B1). A Gaussian coordinate has excess kurtosis $0$.}
\label{tab:kurtosis}
\end{table}

The argument itself is short enough to state as a remark, and its constants are far too loose to deserve the weight of a numbered proposition. Under (B1), coordinate-wise sub-Gaussian concentration with a union bound over the $d$ dimensions bounds the standardized mean error per coordinate by $O(\sqrt{\log(d/\delta)/N})$, and the variance estimate concentrates at the same rate. Expanding the squared score and bounding the cross term by Cauchy--Schwarz then bounds the per-trajectory score error by $\varepsilon_N=O(\sqrt{\log(d/\delta)/N})$ with probability $1-\delta$, with the error largest on the most anomalous trajectories. AUROC is a rank statistic, so it moves only through trusted--failing pairs whose ordering flips. A pair can flip only if its two population scores lie within twice the largest per-trajectory error $\Delta_N$ of each other. Under (B2) the probability of such a pair is at most $2B_\rho\Delta_N$, so $\lvert\mathrm{AUROC}_N-\mathrm{AUROC}_\infty\rvert\le\eta$ holds with probability $1-\delta$ once $N\gtrsim\log(d/\delta)/\eta^2$, with the constant improving with the trusted-versus-failing score margin.

The measured size sweep of Appendix~\ref{app:bensize} supports this simple scaling view.
As $N$ falls from $21$ to $3$, the noise proxy $\sqrt{\log d/N}$ (with $\log d\approx8.3$) grows by about a factor of $2.6$, from $0.63$ to $1.66$. The measured internal AUROC drops only $0.039$ over the same range.
For $d{=}4096$ and $\delta{=}\eta{=}0.05$, the bound with $C{=}1$ gives $N\ge\log(d/\delta)/\eta^2\approx 4500$, about two orders of magnitude above the roughly $10$ to $20$ successful runs the detector actually needs.
A guarantee that loose is not operational, so this stands as a scaling argument rather than a
sample-complexity result.
It predicts the shape of the degradation, error falling as $1/\sqrt{N}$ and growing only
logarithmically in dimension, not how much trusted data a deployment needs.
No constant is fitted, since doing so on four points would dress a shape argument as a calibrated bound.

\subsection{SILENT inference procedure}
\label{app:algo}
Algorithm~\ref{alg:silent} is the deployment procedure, where the scored trajectory $\tau$ is new and
therefore not part of the trusted set $\mathcal{B}$.  Otherwise it contributes to the mean and standard deviation it is
measured against, which inflates AUROC. All reported numbers use that leave-one-out form.

\begin{algorithm}[h]
\small
\caption{SILENT label-free inference}
\label{alg:silent}
\begin{algorithmic}[1]
\Require benign (successful) trajectories $\mathcal{B}$; new trajectory $\tau$; FPR target $\alpha$
\State $\mu,\sigma \gets$ per-dimension mean and standard deviation of step activations over $\mathcal{B}$
\For{each signal $s\in\{\text{internal},\text{tool},\text{text}\}$}
  \State $\bar s_{\text{ben}},\hat\sigma_{\text{ben}} \gets$ mean and standard deviation of $s$ over $\mathcal{B}$
\EndFor
\State $z_t \gets \lVert (h_t-\mu)/\sigma\rVert_2/\sqrt{d}$ for each step $t$ of $\tau$
\State internal $\gets \max_t z_t$; tool, text $\gets$ signals of $\tau$
\State $D(\tau)\gets\sum_s (s-\bar s_{\text{ben}})/\hat\sigma_{\text{ben}}$ \Comment{drop internal if API-only}
\State $\theta\gets (1{-}\alpha)$-quantile of $\{D(b):b\in\mathcal{B}\}$
\State \textbf{flag} if $D(\tau)>\theta$; \textbf{decisive step} $\gets\arg\max_t z_t$
\end{algorithmic}
\end{algorithm}

\end{document}